\documentclass[twoside,11pt]{article}

\usepackage{blindtext}

\usepackage{amssymb}
\usepackage{url}
\usepackage{makecell}
\usepackage{bbding}

\usepackage{booktabs}
\usepackage{amsthm}
\usepackage{multirow}
\usepackage{subfigure}
\usepackage{graphicx}
\usepackage{xcolor}
\usepackage{algpseudocode}
\usepackage{algorithm}
\usepackage{pythonhighlight}

\usepackage{jmlr2e}

\newif\ifcolorfootnotemark
\colorfootnotemarkfalse
\NewCommandCopy{\originalfootnote}{\footnote}
\RenewDocumentCommand{\footnote}{o +m}{%
  \colorfootnotemarktrue
  \IfNoValueTF{#1}{\originalfootnote{#2}}{\originalfootnote[#1]{#2}}%
  \colorfootnotemarkfalse
}
\AddToHook{cmd/@footnotemark/before}{%
  \ifcolorfootnotemark\begingroup\color{red}\fi}
\AddToHook{cmd/@footnotemark/after}{%
  \ifcolorfootnotemark\endgroup\fi}

\makeatletter
\RenewDocumentCommand{\thanks}{+m}{%
  \footnotemark
  \protected@xdef\@thanks{\@thanks
    \protect\footnotetext[\the\c@footnote]{\protect\hypertarget{\Hy@footnote@currentHref}{}#1}}%
}
\makeatother

\DeclareFontShape{OT1}{cmr}{bx}{sc}{<->ssub*cmr/bx/n}{}

\usepackage{amsmath,amsfonts,bm}

\def\eqref#1{Eq.~(\ref{#1})}
\def\ceil#1{\lceil #1 \rceil}

\def\1{\bm{1}}

\def\eps{{\epsilon}}

\def\vone{{\bm{1}}}

\def\mI{{\bm{I}}}

\DeclareMathAlphabet{\mathsfit}{\encodingdefault}{\sfdefault}{m}{sl}
\SetMathAlphabet{\mathsfit}{bold}{\encodingdefault}{\sfdefault}{bx}{n}

\def\gB{{\mathcal{B}}}

\def\gD{{\mathcal{D}}}

\def\gF{{\mathcal{F}}}
\def\gG{{\mathcal{G}}}
\def\gH{{\mathcal{H}}}

\def\gL{{\mathcal{L}}}

\def\gR{{\mathcal{R}}}
\def\gS{{\mathcal{S}}}
\def\gT{{\mathcal{T}}}

\def\gX{{\mathcal{X}}}
\def\gY{{\mathcal{Y}}}

\def\sP{{\mathbb{P}}}

\def\sR{{\mathbb{R}}}

\newcommand{\E}{\mathbb{E}}

\DeclareMathOperator*{\argmax}{arg\,max}
\DeclareMathOperator*{\argmin}{arg\,min}

\def\Dpos{\mathcal{D}_{\mathrm{pos}}}
\def\Dneg{\mathcal{D}_{\mathrm{neg}}}

\def\Lcon{L_{\mathrm{con}}}
\def\Lsup{L_{\mathrm{sup}}}
\def\Lconsup{L_{\mathrm{con}}^{\mathrm{sup}}}

\def\rT{\rho_{\mathcal{T}}}

\def\Disc{\mathrm{Disc}}
\def\Gen{\mathrm{Gen}}

\def\iidsim{\overset{\text{i.i.d}}{\sim}}

\usepackage[capitalize,noabbrev]{cleveref}

\usepackage{lastpage}
\jmlrheading{volume}{year}{pages}{date submitted}{date published}{paper id}{Jialiang Wang, Qing Ma, Xiong Zhou, Gangfeng Hu, Haoliang Li and Xianming Liu}

\ShortHeadings{Internal Discrepancy for Contrastive Learning}{Wang, Ma, Zhou, Hu, Li, Liu}
\firstpageno{1}

\begin{document}
\begin{sloppypar}

% \title{Contraction and Alienation: Towards Theoretical Understanding Non-Contrastive Learning with Neighbor-Averaging Markov Chains}

\title{Internal Discrepancy for Contrastive Learning: Generalization Bounds and Downstream Behaviors}

\author{\name Jialiang Wang$^{1 2}$\thanks{\protect\hypertarget{fn:equal-contribution}{}Equal contribution}\email cswjl@stu.hit.edu.cn \\
        \name Qing Ma$^{2}$\hyperlink{fn:equal-contribution}{\footnotemark[1]}\email qingma2016@gmail.com \\
        \name Xiong Zhou$^1$\email cszx.hit@gmail.com \\
        \name Gangfeng Hu$^1$ \email csgf@hit.edu.cn\\
        \name Haoliang Li$^2$ \email haoliang.li@cityu.edu.hk\\
        \name Xianming Liu$^1$\thanks{Corresponding author} \email csxm@hit.edu.cn\\
        \addr $^1$Harbin Institute of Technology, Harbin, China \\
        \addr $^2$City University of Hong Kong, Hong Kong, China
        }
\editor{My editor}

\maketitle

% \begin{abstract}%   <- trailing '%' for backward compatibility of .sty file
% % \blindtext
% \input{0_abstract}
% \end{abstract}
\begin{abstract}

Contrastive learning has emerged as a popular paradigm for training representation models. While prior studies pursue more practical and less restrictive assumptions to enhance theoretical understanding of contrastive learning, they essentially describe the interconnections among
representations within the data or induced by the function class. To encapsulate these assumptions, we propose the measure of internal discrepancy that quantifies the difference between the mean representation of positive samples and latent cluster centers. We then prove an internal discrepancy-aware generalization bound for downstream supervised tasks following contrastive learning, without requiring additional assumptions. This generalization bound not only serves as a unifying framework connecting various assumptions but also motivates us to further reduce the internal discrepancy, thereby introducing the relaxed k-means loss as an approximation. Moreover, we investigate the downstream behaviors of contrastive learning from the perspective of contrastive learning-incurred supervised losses. In the presence of class imbalance, we shed light on why contrastive learning can facilitate imbalanced learning. Finally, we investigate the optimality of these incurred supervised losses, proving several intriguing properties.

\end{abstract}


\section{Introduction}
In the past decade of machine learning research, many classification models have relied heavily on the availability of a large amount of labeled data \citep{krizhevsky2012imagenet, he2016deep, vaswani2017attention}. However, it is not always possible to obtain sufficient labeled data for scaling up models, prompting researchers to leverage vast amounts of unlabeled data \citep{misra2020self, grill2020bootstrap, goyal2021self, tomasev2022pushing, he2022masked, kong2023understanding, dehghani2023scaling, zhang2025augmentation}, with self-supervised learning emerging as a promising path to advance machine learning \citep{balestriero2023cookbook}. In particular, we focus on contrastive learning that captures semantically relevant and discriminative information from the data by bringing the representations of positive pairs closer together while pushing the representations of negative samples apart \citep{gutmann2012noise, he2020momentum, tian2020contrastive, Caron_2021_ICCV, hieu2025generalization, rusak2025infonce}.

Despite the empirical success of contrastive learning \citep{mikolov2013distributed, logeswaran2018an, radford2021learning, yu2022coca, xu2023contrastive}, there remains a lack of comprehensive understanding regarding how this paradigm effectively learns meaningful representations that facilitate downstream tasks \citep{tian2020makes, garrido2023on, chen2024your, tan2024contrastive, wang2025linear}. For instance, \citet{kang2021exploring} observed that contrastive learning can achieve a balanced representation space in the presence of class imbalance, \textit{but there has been limited theoretical analysis of this phenomenon}. Due to the absence of supervision, contrastive learning primarily relies on the intrinsic structure within the data \citep{tian2020makes, li2026consistency}, and the inductive biases of the function class and training algorithm \citep{wen2021toward, saunshi2022understanding, haochen2023theoretical, parulekar23infonce}. However, without further assumptions on the data or inductive biases, relying solely on contrastive losses may not adequately establish relevant associations for downstream behaviors.

Prior studies have introduced various data assumptions to bridge this gap, including the (nearly) conditional independence between positive samples \citep{saunshi2019theoretical, lee2021predicting, zou2023generalization}, the overlap of augmentation sets for semantically similar samples \citep{haochen2021provable, wang2022chaos}, the expansion or concentration within augmentation sets \citep{wei2021theoretical, huang2023towards}, and the label consistency of positive pairs \citep{zhang2025augmentation}. While these studies pursue more practical and less restrictive data assumptions to understand contrastive learning, \citet{saunshi2022understanding} pointed out that data assumptions alone, even with a small contrastive loss, are insufficient to capture the success of contrastive learning in realistic scenarios. They further argue that understanding contrastive learning requires incorporating inductive biases of the function class. This has motivated studies to investigate constrained function classes, such as assumptions regarding the ability to implement stronger inner-cluster connections \citep{haochen2023theoretical}, cluster-preserving representations \citep{parulekar23infonce}, and centered representations \citep{cui2025augmentation}. In light of these two lines of research, \textbf{several critical questions arise}:
\begin{enumerate}
    \item \textit{Do prior data assumptions that appear distinct essentially characterize or constrain a common aspect? Is there a measure that connects these data assumptions?}
    \item \textit{Can this measure effectively describe the inductive biases of the function class or training algorithm? }
    \item \textit{Can we further establish a generalization bound based on this measure without requiring additional assumptions, thus providing a unified framework for studying both data assumptions and inductive biases?}
    \item \textit{Can we further explore the downstream behaviors of contrastive learning, including its handling of class imbalance and its optimality?}
\end{enumerate}
We answer these questions affirmatively. The core insight lies in the observation that prior assumptions essentially or partially describe the interconnections among representations of samples, augmentations, and latent clusters. To quantify these interconnections, we propose the measure of \textit{internal discrepancy}, denoted by $\Disc(f)$, which characterizes the discrepancy between the mean representation of positive samples and latent cluster centers. Without relying on specific assumptions, we prove an internal discrepancy-aware generalization bound on downstream supervised performance for contrastive learning, expressed informally as $\Lsup(\widehat{f})\lesssim\Lcon(f)+\Disc(\widehat{f})+\Gen_M$ for any $f\in\gF$. Here, $\Lsup(\widehat{f})$ is a supervised loss w.r.t. the empirical minimizer $\widehat{f}$, $\Lcon(f)$ represents the contrastive loss, and $\Gen_M$ denotes the generalization error. This bound suggests that downstream performance will be determined by both the value of the contrastive loss and the internal discrepancy. By further establishing the association between specific assumptions and the internal discrepancy, we can readily develop theoretical insights into downstream tasks with those assumptions. Replacing $\Disc(\widehat{f})$ with the intra-class discrepancy of $f$ and the positive-pair alignment of $\widehat{f}$ in the bound motivates us to minimize an approximation of the internal discrepancy, namely \textit{the relaxed k-means loss}, which aims to reduce the distance between representations within a truncated range.

Moreover, the supervised loss $\Lsup(\widehat{f})$ corresponds to a specific supervised loss function, which we refer to as a \textit{contrastive learning-incurred supervised loss}. Based on the incurred supervised loss, we provide insights into the downstream behaviors of contrastive learning, specifically through the InfoNCE loss \citep{oord2018representation} and the spectral contrastive loss (SpeCL) \citep{haochen2021provable}. In class-imbalanced scenarios, we prove that the InfoNCE-incurred loss approximates the balanced softmax loss \citep{ren2020balanced} that accommodates class distribution shifts to mitigate class imbalance, and that the cosine similarities between representations learned by the SpeCL-incurred loss remain consistent regardless of the class distribution. These results help explain why contrastive learning can facilitate class-imbalanced learning. Furthermore, we explore the optimality of these contrastive learning-incurred supervised losses, proving that representations learned under the InfoNCE-incurred loss exhibit variability collapse and converge to a simplex equiangular tight frame. Additionally, the SpeCL-incurred loss leads to optimal features that achieve the best approximation of the identity matrix.

Our main contributions are as follows:
\begin{itemize}
    \item We introduce internal discrepancy, a measure that characterizes the discrepancy between the mean representation of positive samples and latent cluster centers and can be associated with various data assumptions.
    \item We derive the internal discrepancy-aware generalization bound on downstream performance for contrastive learning without requiring assumptions on the data. The generalization bound motivates us to propose the relaxed k-means loss to approximately minimize the internal discrepancy.
    \item We investigate the role of contrastive learning in mitigating class imbalance, revealing that the InfoNCE-incurred loss approximates the balanced softmax loss and that class imbalance only scales representations under the SpeCL-incurred loss.
    \item We prove the optimality of contrastive learning through its incurred supervised losses, including variability collapse and convergence to a simplex equiangular tight frame for the InfoNCE-incurred loss, and the best approximation to the identity matrix by optimal representations for the SpeCL-incurred loss.
\end{itemize}

\section{Preliminaries}
\paragraph{Notation.}
Let $\gX$ denote the set of all possible data points, $\gY$ be the total label space, and $\gF$ be the function class of norm-bounded representation functions $f:\gX\rightarrow \sR^d$ s.t. $\|f(\cdot)\|\le R$ for some $R>0$. We define $\gD$ as the distribution of samples\footnote{For ease of exposition, we consider a distribution over original samples and their augmentations.}, $\gD^c$ as its component w.r.t. class $c$, and $\gD_{\gT}$ as the joint distribution of samples and their corresponding labels within the (latent) class space $\gT\subseteq \gY$. Let $\rT$ be the class distribution of $\gT$ and $|\gT|$ be the cardinality of $\gT$. Throughout, we assume $|\gT|\ge2$ and $\rT(c)>0$ for all $c\in\gT$. We further assume that each sample $x\sim\gD$ has a deterministic label $y$, and write $(x,y)\sim\gD_{\gT}$ for brevity.

\paragraph{Supervised Loss.} 
The loss incurred by the classifier $g:\gX\rightarrow \sR^{|\gT|}$ on a point $(x,y)\sim \gD_{\gT}$ is defined as $h\left(g(x)_y,\{g(x)_{c}\}_{c\in\gT,c\neq y}\right)$, typically tailored for enlarging the prediction $g(x)_y$ of the target class while reducing the predictions for non-target classes $\{g(x)_c\}_{c\neq y}$. Consequently, the supervised loss w.r.t. the classifier $g$ is
\begin{equation}
L_{\mathrm{sup}}(\gD_{\gT},g,h)=\E_{(x,y)\sim\gD_{\gT}}[h(g(x)_y,\{g(x)_{c}\}_{c\in\gT,c\neq y})].
\end{equation}

\paragraph{Linear Probing.}
To evaluate the downstream classification performance of the representation model $f$, we train a linear classifier $W\in\sR^{|\gT|\times d}$ and use $g(x)=Wf(x)$ to calculate the classification loss $\Lsup(\gD_{\gT}, Wf,h)$. By fixing $f$ and optimizing the linear classifier, we use a slight abuse of notation and define the supervised loss of $f$ on $\gD_{\gT}$ as the infimum over $W$:
\begin{equation}
L_{\mathrm{sup}}(\gD_{\gT}, f,h)=\inf_{W\in\sR^{|\gT|\times d}} L_{\mathrm{sup}}(\gD_{\gT}, Wf,h).
\end{equation}
Since $L_{\mathrm{sup}}(\gD_{\gT}, f,h)\le L_{\mathrm{sup}}(\gD_{\gT}, Wf,h)$ for any $W$, we focus on bounding a surrogate classification loss with the mean classifier that comprises class means of representations \citep{saunshi2019theoretical, wang2022chaos, huang2023towards, joshi2023data}. Throughout this paper, unless otherwise specified, 
$W$ is the mean classifier, i.e., $w_c=\E_{x\sim\gD^c}f(x)$ for all $c\in\gT$.

\paragraph{Contrastive Learning.}
% \section{Internal-Discrepancy-Aware Generalization Bound}
Contrastive learning is based on the fundamental idea of bringing representations of positive pairs closer while pushing representations of negative samples apart. 
We model an anchor $x$ and its positive sample $x^+$ by the joint distribution $\Dpos(x,x^+)=\gD(x)\Dpos(x^+|x)$\footnote{In supervised contrastive learning, $\Dpos(x^+|x)$ specifically includes samples that belong to the same class as $x$. In the unsupervised setting, $\Dpos(x^+|x)$ covers the distribution of random views (or augmentations) of $x$.}, and $k\ge1$ negative samples $\{x_i^-\}_{i=1}^k$ are drawn i.i.d. from $\Dneg(x^-|x)$\footnote{For supervised contrastive learning, $\Dneg(x^-|x)$ is characterized by $\gD^{\not y_x}(x^-)$, which represents samples from classes other than the class of $x$. In the unsupervised scenario, $\Dneg(x^-|x)=\gD(x^-)$ describes the overarching data distribution, encompassing all classes without using label information.}. The population contrastive loss is defined as   
\begin{equation}
\label{contrastive-loss}
\Lcon(f,\gD;\ell)=\E_{\substack{x,x^+\sim \Dpos\\\{x_i^-\}_{i=1}^k\iidsim \Dneg}}\ell\left(f,x,x^+,\{x_i^-\}_{i=1}^k\right),
\end{equation}
where $\ell$ represents the contrastive loss function. For the sake of brevity, we will use $\ell\left(f,x,x^+,\{x_i^-\}_{i=1}^k\right)$ to denote $\ell\left(f(x)^\top f(x^+),\{f(x)^\top f(x_i^-)\}_{i=1}^k\right)$.  Additionally, we might abbreviate $\Lcon(f,\gD;\ell)$ as $\Lcon(f;\ell)$ or $\Lcon(f)$ in certain contexts. Intuitively, the essence of minimizing the contrastive loss lies in amplifying the similarity between positive-pair representations $f(x)$ and $f(x^+)$, while striving to separate the representations of negative samples. Correspondingly, \citet{wang2020understanding} have identified two properties related to this process: the alignment of positive pairs and the uniformity of features.

We exemplify two contrastive loss functions below. The first is the widely used InfoNCE loss:
\begin{example}[InfoNCE]
The InfoNCE loss in contrastive learning \citep{oord2018representation, chen2020simple, he2020momentum} is defined as $\ell(f,x,x^+,\{x_i^-\}_{i=1}^k)=\log \left[1+\sum_{i=1}^k \exp(f(x)^\top(f(x_i^-)-f(x^+)))\right]$.
\end{example}

The second is the spectral contrastive loss that is both sample-contrastive and dimension-contrastive \citep{haochen2021provable, haochen2023theoretical, haochen2022beyond, balestriero2022contrastive, zhou2024zero}, acting as a well-established theoretical surrogate for the InfoNCE loss:

\begin{example}[Spectral Contrastive Loss (SpeCL)]
The spectral contrastive loss \citep{haochen2021provable} is formulated as $\ell(f,x,x^+,\{x_i^-\}_{i=1}^k)=-2f(x)^\top f(x^+)+\frac{1}{k}\sum_{i=1}^k \left(f(x)^\top f(x_i^-)\right)^2$.
\end{example}
As detailed in Appendix \ref{appendix:convexity-lipschitzness}, the separate convexity of the InfoNCE loss and the spectral contrastive loss can be easily verified w.r.t. the representations $f(x)$, $f(x^+)$, and $f(x_i^-)$. Furthermore, these two loss functions are also Lipschitz continuous when $f$ adheres to bounded norm constraints. Therefore, we consider contrastive loss functions with both convexity and Lipschitzness.

For the tuples $\left\{\left(x_j,x_j^+,\{x_{ji}^-\}_{i=1}^k\right)\right\}_{j=1}^M$ drawn i.i.d. according to the sampling scheme above, the empirical counterpart is defined as
\begin{equation}
\widehat{L}_{\mathrm{con}}(f)=\frac{1}{M}\sum\nolimits_{j=1}^M\ell\left(f,x_j,x_j^+,\{x_{ji}^-\}_{i=1}^k\right).
\end{equation}

The goal of contrastive learning is to identify the risk minimizer $f^*\in\argmin_{f\in\gF}\Lcon(f,\gD)$. In practical scenarios, where access is limited to a finite number of samples, our focus pivots towards finding the empirical risk minimizer $\widehat{f}\in \argmin_{f\in\gF}\widehat{L}_{\mathrm{con}}(f)$. As in supervised learning, the generalization bound for contrastive learning is measured by the difference between $\widehat{f}$ and $f^*$, but with a greater emphasis on bounding the downstream supervised loss $\Lsup(\gD_{\gT},W\widehat{f},h)$ for some $W$.

% The discrepancy between empirical risk minimizer and 

\section{Internal Discrepancy for Contrastive Learning}

In this section, we begin by introducing the concept of internal discrepancy and its relationship to various assumptions outlined in the literature. Following this, we establish an internal discrepancy-aware generalization bound for contrastive learning, which does not rely on any specific data assumptions. Moreover, we propose the relaxed k-means loss as a means to approximate the minimization of internal discrepancy. Additional results for scenarios involving a single negative sample are provided in Appendix \ref{appendix:more-results}.
% In this section, we first present the definition of internal discrepancy and its association to several data assumptions in prior studies. Subsequently, we prove the internal discrepancy-aware generalization bound for contrastive learning, which does not require any data assumptions, and further introduce the relaxed k-means loss to approximately minimize internal discrepancy. We also provide some additional results for scenarios involving a single negative sample and supervised contrastive learning in Appendix \ref{appendix:more-results}.

\subsection{Internal Discrepancy}
\label{sec:internal-discrepancy}
Despite the empirical success of contrastive learning, a significant gap persists in theoretical understanding of how this paradigm effectively learns meaningful relationships between data points. Due to the absence of labels in the data, contrastive learning primarily relies on the associations within the data and their augmentations, and the inductive biases of the function class. Prior studies have presented various data assumptions to establish relevant theories, such as conditional independence of positive pairs \citep{saunshi2019theoretical, lee2021predicting, zou2023generalization} and the overlap of augmentation sets \citep{haochen2021provable, wang2022chaos}. They have also considered assumptions on the function class, including centered representations \citep{cui2025augmentation}. While diverse, these assumptions essentially or at least partially describe interconnections between samples, augmentations, and latent clusters.

To essentially measure these interconnections, we define the internal discrepancy as follows:
\begin{definition}[Internal Discrepancy]
For the data distribution $\gD$ with latent classes $\gT$, the internal discrepancy is defined as follows:
\begin{equation}
\begin{aligned}
\Disc(f)&=\E_{\substack{c\sim \rT\\ x\sim\gD^c}}\left\|\E_{x^+\sim\Dpos(\cdot|x)}f(x^+)-\E_{x\sim\gD^c}f(x)\right\|,\\
\end{aligned}
\end{equation}
where $\|\cdot\|$ denotes the $\ell_2$-norm.
\end{definition}

\begin{figure}[t]
    \centering
    \includegraphics[width=0.98\linewidth]{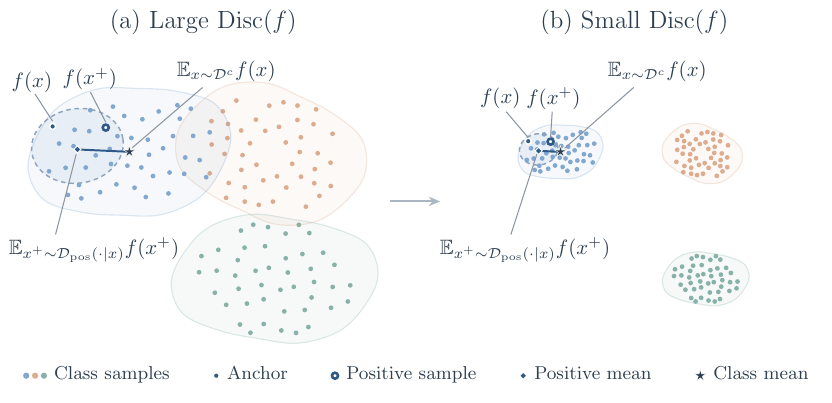}
    \caption{Schematic comparison of large and small internal discrepancy. Colors denote latent classes, while dashed regions highlight local positive pairs. The filled and open circles indicate $f(x)$ and $f(x^+)$. The diamond and star mark $\E_{x^+\sim\Dpos(\cdot\mid x)}f(x^+)$ and $\E_{x\sim\gD^c}f(x)$, respectively. The expected distance between the conditional positive mean and the class mean equals $\Disc(f)$. The right panel depicts more compact class and local regions, accompanied by a shorter segment in the illustrated configuration.}
    \label{fig:internal-discrepancy-contraction}
\end{figure}

Internal discrepancy quantitatively characterizes the alignment between the mean representation of positive samples of a data point and the classwise mean representation. Moreover, we have:
\begin{equation}
\begin{aligned}
\Disc(f)\le &\underbrace{\E_{x,x^+\sim\Dpos(x,x^+)}\|f(x^+)-f(x)\|}_{\text{positive-pair alignment }\Disc_a(f)}+ \underbrace{\E_{\substack{c\sim \rT\\ x\sim\gD^c}}\|f(x)-\E_{x\sim\gD^c}f(x)\|}_{\text{intra-class discrepancy }\Disc_i(f)},
\end{aligned}
\end{equation}
which implies that internal discrepancy is further governed by two key components: the positive-pair alignment term $\Disc_a(f)$, also referred to as the alignment loss \citep{wang2020understanding}, and the intra-class discrepancy $\Disc_i(f)$ that measures the discrepancy between the representations of individual samples and their respective latent class means. While prior works often posit specific assumptions that appear distinct, we demonstrate that they essentially serve as conditions to constrain the value of internal discrepancy $\Disc(f)$ or the two terms in its upper bound. Figure~\ref{fig:internal-discrepancy-contraction} illustrates a configuration in which more compact class clouds and local positive-pair regions accompany a smaller distance between the conditional positive mean and the class mean, reducing internal discrepancy.

% While prior works often posits specific assumptions about the alignment loss and intra-class discrepancy, we elaborate on the association of internal discrepancy $\Disc(f)$ with these assumptions.

\paragraph{Association with Previous Assumptions.} The internal discrepancy metric provides a novel perspective for integrating and expanding upon assumptions made in prior literature regarding data characteristics \citep{saunshi2019theoretical, wang2022chaos, huang2023towards, zhang2025augmentation} and the function class \citep{haochen2023theoretical, parulekar23infonce, cui2025augmentation}:

(i) The assumption of \textit{conditional independence} of positive-pair draws \citep{saunshi2019theoretical} states that $\Dpos(x,x^+|c)=\gD^c(x)\gD^c(x^+)$. This leads to a zero internal discrepancy ($\Disc(f)=0$), as $\E_{x^+\sim\Dpos(\cdot|x)}f(x^+)=\E_{x\sim\gD^c}f(x)$.

(ii) The \textit{perfect alignment} assumption between positive pairs (or the augmentation-invariant assumption) \citep[Assumption 4.6]{wang2022chaos} implies $f(x)=f(x^+)$ for all $x,x^+\sim\Dpos(x,x^+)$, that is, $\Disc_a(f)=0$. Moreover, the internal discrepancy $\Disc(f)$ simplifies to the intra-class discrepancy $\E_{\substack{c\sim \rT\\ x\sim\gD^c}}\left\|f(x)-\E_{x\sim\gD^c}f(x)\right\|$, quantified as the expected distance between each sample and its corresponding class mean.

(iii) The \textit{$(\sigma,\delta)$-augmentation} described by \citet[Definition 1]{huang2023towards} means that for each class, whose sample set is $G_c$, there exists a subset $G_c^0\subseteq G_c$ such that $\sP[x\in G_c^0]\ge \sigma\sP[x\in G_c]$ with $\sigma\in(0,1]$ and $\sup_{x_1,x_2\in G_c^0}d_A(x_1,x_2)\le \delta$, where $d_A(x_1,x_2)=\min_{x_1'\in A(x_1),x_2'\in A(x_2)} \|x_1'-x_2'\|$ for augmentation $A$. Under the additional assumptions of $\eps$-close representations among augmented data ($\|f(x_1)-f(x_2)\|\le \eps$ for all $x_1,x_2\in A(x)$) and $\iota$-Lipschitzness of $f$ in \citet{huang2023towards}, we obtain $\Disc(f)\le\sigma^2(\iota\delta+2\eps)+2R(1-\sigma^2)$.

(iv) The \textit{$\gF$-implementable inner-cluster connection larger than $\beta$} assumption \citep[Assumption 3.3]{haochen2023theoretical} implies $\E_{x,x^+\sim\Dpos^c}\|f(x^+)-f(x)\|^2\ge\beta\E_{x\sim\gD^c}\|f(x)-\E_{x\sim\gD^c}f(x)\|^2$ for any $f\in\gF$. Then we can derive that $\Disc(f)\le\Disc_a(f)+\frac{1}{\sqrt{\beta}}\sqrt{\E_{c\sim\rT}\E_{x,x^+\sim\Dpos^c}\|f(x^+)-f(x)\|^2}$, which measures the divergence between representations of positive pairs.

(v) For the \textit{cluster-preserving} function class \citep[Definition 4.1]{parulekar23infonce} satisfying $f(x)=f(x')$ for all $x,x'\iidsim\gD^c$, $c\in\gT$ and $f\in\gF$, $\Disc(f)=\E_{x\sim\gD}\|f(x)-\E_{x^+\sim\Dpos(\cdot|x)}f(x^+)\|\le \Disc_a(f)$, which means that the internal discrepancy is bounded by the positive-pair alignment term.

(vi) For the \textit{label consistency} assumption \citep[Assumption 4]{zhang2025augmentation}, positive pairs satisfy $\sP_{x,x^+\sim\Dpos}[y_x\neq y_{x^+}]\le\alpha$ for some $\alpha\in[0,1]$, which yields $\Disc(f)\le\Disc_i(f)+2R\alpha$ when $x^+$ also follows the marginal distribution $\gD$. This implies that the internal discrepancy is bounded by the intra-class discrepancy and an additional term accounting for positive pairs from different classes. Moreover, the connected intra-class augmentation graphs and $\eps$-alignment conditions in \citet[Assumptions 8 and 9]{zhang2025augmentation} provide a means to control the intra-class discrepancy by propagating alignment along paths within each class.

(vii) For the \textit{centered representation} assumption \citep[Assumption 2.2]{cui2025augmentation}, the original sample satisfies $\E_{x^+\sim\Dpos(\cdot|x)}f(x^+)=f(x)$. When $\gD^c$ denotes the distribution of original samples in class $c$ and $\Dpos(\cdot|x)$ is induced by augmenting the original anchor $x$, we obtain $\Disc(f)=\Disc_i(f)$ directly from the definition. This assumption constrains the mean representation of augmented samples while allowing $\Disc_a(f)>0$, so that the internal discrepancy reduces to the intra-class discrepancy without requiring the perfect alignment in (ii).

% A zero internal discrepancy does not require conditional independence. For example, let $x=(y,u)$, where $y$ is uniform on $\{-1,1\}$ and $u$ is independently uniform on $\{-1,0,1\}$. Positive pairs preserve $y$, with $u^+=0$ when $u=0$ and $u^+$ uniform on $\{-1,1\}$ otherwise. The two marginals agree, but positive pairs are not conditionally independent within each class. For $f(y,u)=(ay,bu)$ with $a,b>0$ and $a^2+b^2\le R^2$, the identity $\E[u^+|u]=0$ gives $\Disc(f)=0$, while $\Disc_a(f)=\Disc_i(f)=2b/3$.

In essence, by formulating a generalization bound that takes into account internal discrepancy as the intermediate bridge, we can readily use specific data assumptions to develop theoretical insights into downstream tasks. This only requires us to establish the association between internal discrepancy and these data assumptions.

\subsection{Internal Discrepancy-Aware Generalization Bound}
\label{sec:contrastive-generalization-bound}

We first show that there exists a supervised loss bounded by the population contrastive loss and internal discrepancy:
\begin{theorem}
\label{thm:bound-for-contrastive-losses}
For the loss function $\ell$ that is convex with respect to each of $f(x^+)$ and $f(x_i^-)$ and $\eta$-Lipschitz w.r.t. the positive representation $f(x^+)$, and the distribution $\gD$ with latent classes $\gT$, for all $f\in\gF$, there exist a linear classifier $W\in\sR^{|\gT|\times d}$ and a supervised loss $h:\sR^{|\gT|}\times\gT\rightarrow \sR$ such that
\begin{equation}
    \Lsup(\gD_{\gT},Wf,h)\le \Lcon(f;\ell)+\eta\cdot\Disc(f),
\end{equation}
where equality holds if $f(x)=f(x^+)=\E_{x'\sim\gD^c}f(x')$ for all $x,x^+\sim\Dpos^c$ and $c\in\gT$. And $h(w_y^\top f(x),\{w_c^\top f(x)\}_{c\neq y})=\E_{\{c_i\}_{i=1}^k\iidsim\rho_{\gT}}\ell(f(x),w_y, \{w_{c_i}\}_{i=1}^k)$ is the incurred supervised loss.
\end{theorem}

The following proposition shows that the internal discrepancy term in Theorem \ref{thm:bound-for-contrastive-losses} cannot be dropped, even for a near-minimizer of the contrastive loss:
\begin{proposition}
\label{prop:tightness}
For any $k\ge1$, class distribution $\rT$, and $\eps>0$ with $d\ge|\gT|\ceil{k/(2\eps)}$, there exist a distribution $\gD$ with latent classes $\gT$ and a representation $f$ with $\|f(x)\|=R$ such that, for the InfoNCE loss with $\eta=R$, positive samples $x^+=x$, and negative samples drawn from $\gD$, we have $\Disc(f)=R$, the mean classifier $W=0$, $\Lcon(f)\le\inf_{f'}\Lcon(f')+\eps$ with the infimum over all $f'$ satisfying $\|f'(x)\|\le R$, and
\begin{equation}
    \Lsup(\gD_{\gT},Wf,h)=\log(1+k)\ge\Lcon(f)+\eta\cdot\Disc(f)-\log\frac{e^{R^2}+k}{1+k}-\eps.
\end{equation}
\end{proposition}
Since $\log\frac{e^{R^2}+k}{1+k}\le\frac{e^{R^2}-1}{1+k}$, the bound in Theorem \ref{thm:bound-for-contrastive-losses} becomes tight in this example as $k$ grows, with the gap $\Lsup(\gD_{\gT},Wf,h)-\Lcon(f)$ approaching $\eta\cdot\Disc(f)$.

\begin{figure}[!t]
    \centering
    \subfigure[CIFAR-10]{
    \label{fig:cifar10}
    \includegraphics[width=1.8in]{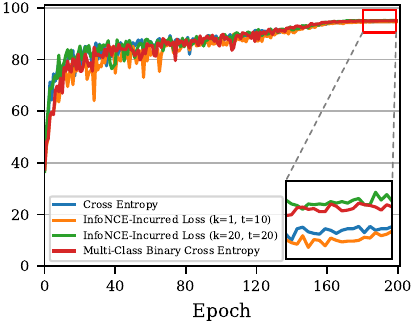}
    }
    \subfigure[CIFAR-100]{
    \label{fig:cifar100}
    \includegraphics[width=1.8in]{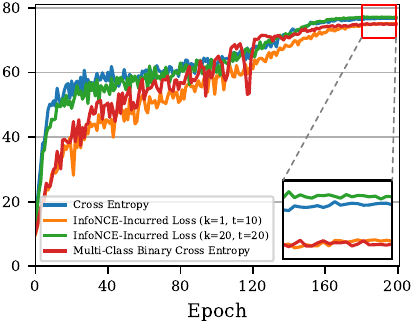}
    }
    \subfigure[ImageNet-100]{
    \label{fig:imagenet100}
    \includegraphics[width=1.8in]{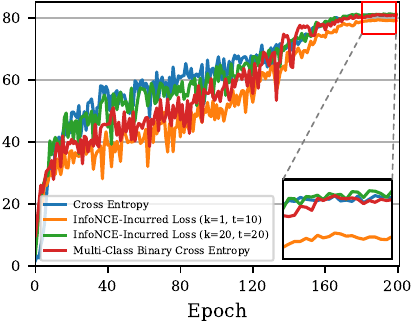}
    }
   
    \caption{Illustration of validation accuracies (\%) on CIFAR-10, CIFAR-100, and ImageNet-100 using cross entropy, InfoNCE-incurred loss as per Theorem \ref{thm:bound-for-contrastive-losses}, and multi-class binary cross entropy defined in Equation \ref{eq:multi-class-binary-ce}. The depicted results demonstrate that these three loss functions deliver comparable performance.}
    \label{fig:ce-infonce-mbce}

\end{figure}

\begin{table}[t]
\centering
\caption{Validation accuracies (\%) on the CIFAR-10 and CIFAR-100 datasets. These results are obtained through supervised learning using cross entropy (CE) or InfoNCE-incurred loss (with the number of negatives $k$ and the number of sampling repetitions $t$). As is evident from the results, the InfoNCE-incurred supervised loss, when optimized with suitable values of $k$ and $t$, can yield performance that is on par with or even surpasses that of the traditional CE approach.}
\label{tab:incurred-CE}

\begin{sc}
\begin{tabular}{c|c|cccc}
\toprule
\multirow{2}{*}{\textbf{Dataset}}
& \multirow{2}{*}{\textbf{Method}}
& \multicolumn{4}{c}{\textbf{Sampling repetitions $t$}} \\
& & $10$ & $20$ & $40$ & $100$ \\
\midrule\midrule
\multirow{6}{*}{\textbf{CIFAR-10}}
& CE      & \multicolumn{4}{c}{94.79} \\
\cmidrule{2-6}
& $k=1$   & 94.79 & 94.93 & 94.83 & 94.91 \\
& $k=10$  & 95.20 & 94.87 & 95.05 & 94.71 \\
& $k=20$  & 94.81 & 94.82 & 94.85 & 94.65 \\
& $k=40$  & 94.76 & 94.73 & 94.79 & 94.73 \\
& $k=100$ & 94.49 & 95.07 & 94.66 & 94.70 \\
\midrule
\multirow{6}{*}{\textbf{CIFAR-100}}
& CE      & \multicolumn{4}{c}{76.37} \\
\cmidrule{2-6}
& $k=1$   & 74.68 & 75.03 & 75.00 & 74.87 \\
& $k=10$  & 76.79 & 76.76 & 76.79 & 76.84 \\
& $k=20$  & 77.18 & 77.53 & 76.74 & 76.82 \\
& $k=40$  & 76.77 & 77.27 & 76.86 & 76.91 \\
& $k=100$ & 76.56 & 76.64 & 76.76 & 76.77 \\
\bottomrule
\end{tabular}
\end{sc}
\end{table}

\paragraph{Contrastive Learning-Incurred Supervised Loss.}
The supervised loss $h$ incurred by contrastive learning involves an expectation over $k$ i.i.d. class draws $\{c_i\}_{i=1}^k$ from $\rT$ and may suffer from class collisions due to the presence of the target-class weight among the negatives. Moreover, the specific form of $h$ is intricately tied to factors such as the number of negatives $k$ and the latent class distribution $\rT$, posing practical challenges in its application. Nevertheless, the incurred loss $h$ can be effective for classification problems. For instance, the InfoNCE-incurred supervised loss can be reformulated as $h(Wf(x),y)=\sum_{j=0}^k\binom{k}{j}(\rT(y))^j(1-\rT(y))^{k-j}\E_{\{c_i\}_{i=1}^{k-j}\iidsim\rT^{\not y}}\log[1+j+\sum_{i=1}^{k-j} e^{f(x)^\top(w_{c_i}-w_y)}]$ (where $\rT^{\not y}$ denotes the class distribution excluding class $y$), which explicitly accounts for class collision. This formulation bears similarity to the cross entropy (CE) loss $\log\left[1+\sum_{c\in\gT,c\neq y}e^{f(x)^\top(w_c-w_y)}\right]$. In Figure \ref{fig:ce-infonce-mbce} and Table \ref{tab:incurred-CE}, we demonstrate that the InfoNCE-incurred loss (we randomly sample $k$ classifier weights $t$ times) can yield comparable results to CE. Additionally, we prove below that they share the same global minimizers under mild conditions. Further analysis of the incurred loss will be developed in Section \ref{sec:incurred-losses}.

\begin{lemma}
\label{lemma:infonce-ce}
Consider the class-balanced distribution and the InfoNCE-incurred supervised risk $\gH(f)=\E_{(x,y)\sim\gD_\gT}\E_{\{c_i\}_{i=1}^k\iidsim\rT}\log\left[1+\sum_{i=1}^{k}e^{f(x)^\top(w_{c_i}-w_y)}\right]$. If $d\ge|\gT|-1$, then $\gH(f)$ and the CE loss $\E_{(x,y)\sim\gD_\gT}\log\left[1+\sum_{c\neq y}e^{f(x)^\top(w_c-w_{y})}\right]$ share the same global minimizers over all representation functions satisfying $\|f(x)\|\le R$.
\end{lemma}

While a variety of works have studied the capacity of contrastive learning to recover meaningful representations or latent clusters under certain data assumptions, we put forth a generalization bound for Lipschitz convex losses, which holds without any additional assumptions.
\begin{theorem}[Generalization Bound]
\label{thm:generalization-bound-for-contrastive-learning}
Under the conditions in Theorem \ref{thm:bound-for-contrastive-losses}, assume further that $\ell$ is $\eta\zeta_1(R)$-Lipschitz w.r.t. $(f(x),f(x^+),\{f(x_i^-)\}_{i=1}^k)$ and has a range of length at most $\zeta_2(R)$. With probability at least $1-\delta$, for all $f\in\gF$, there exist a linear classifier $W\in \sR^{|\gT|\times d}$ and a supervised loss $h$ such that
\begin{equation}
\Lsup(\gD_{\gT},W\widehat{f},h)\le \Lcon(f)+\eta\cdot\Disc(\widehat{f})+\Gen_{M},
\end{equation}
where $\Gen_M=O\left(\frac{\eta\zeta_1(R)\gR_{\gS}(\gF)}{M}+\zeta_2(R)\sqrt{\frac{\log(2/\delta)}{M}}\right)$ and $\gR_{\gS}(\gF)$ is the (unnormalized) empirical Rademacher complexity of $\gF$ on $M$ i.i.d. training tuples $\gS$.
\end{theorem}

\paragraph{Remark.}
Theorem \ref{thm:generalization-bound-for-contrastive-learning} supports the observation of \citet{saunshi2022understanding} that achieving a low contrastive loss alone may not guarantee good downstream performance, as exemplified by Proposition \ref{prop:tightness}. A low contrastive loss and a small internal discrepancy provide sufficient control of the incurred supervised risk. This theorem also highlights several insights: (i) The inductive biases of the function class and optimization algorithm should lead to the learned model exhibiting both a low contrastive loss $\Lcon(\widehat{f})$ and a small internal discrepancy $\Disc(\widehat{f})$. Given that $\Disc(f)\le \Disc_a(f)+\Disc_i(f)$, improving positive-pair alignment and reducing intra-class discrepancy tighten this upper bound. (ii) In alignment with classical learning theory principles, an adequately large number $M$ of independent training tuples aids in reducing $\Gen_M$ when the complexity term grows sublinearly in $M$. This perspective may help explain the empirical effectiveness of rich data augmentation in contrastive learning \citep{chen2020simple, huang2023towards}.
% For InfoNCE with $k=1$, Equation \ref{eq:infonce-classification-bound} in Appendix \ref{appendix:more-results} further bounds the classification error of the mean classifier under the original class prior.

The incurred supervised loss further controls the classification error of the linear classifier and its balanced error that weights all classes equally:
\begin{corollary}
\label{cor:classification-error}
Let $h$ be the InfoNCE-incurred loss in Theorem \ref{thm:bound-for-contrastive-losses}. For $\{c_i\}_{i=1}^k\iidsim\rT$, write $k_c=\sum_{i=1}^k\mathbb{I}(c_i=c)$ for $c\in\gT$, and define $h_0(y)=\E\log(1+k_y)$ and $\omega_k(y)=\min_{c\neq y}\E\log\left(1+\frac{k_c}{1+k_y}\right)$ for $y\in\gT$. Then for any $f$ and $W$, with any fixed rule for breaking ties, the prediction $\hat{y}(x)=\argmax_{c\in\gT}w_c^\top f(x)$ satisfies
\begin{equation}
    \begin{aligned}
        \sP_{(x,y)\sim\gD_{\gT}}\left[\hat{y}(x)\neq y\right]
        &\le\frac{\Lsup(\gD_{\gT},Wf,h)-\E_{y\sim\rT}h_0(y)}{\min_{y\in\gT}\omega_k(y)},\\
        \frac{1}{|\gT|}\sum_{y\in\gT}\sP_{x\sim\gD^y}\left[\hat{y}(x)\neq y\right]
        &\le\frac{\Lsup(\gD_{\gT},Wf,h)-\E_{y\sim\rT}h_0(y)}{|\gT|\min_{y\in\gT}\rT(y)\omega_k(y)}.
    \end{aligned}
\end{equation}
With probability at least $1-\delta$, Theorem \ref{thm:generalization-bound-for-contrastive-learning} further bounds both numerators for the mean classifier of $\widehat{f}$ by $\Lcon(f)+R\Disc(\widehat{f})+\Gen_M-\E_{y\sim\rT}h_0(y)$ for all $f\in\gF$.
\end{corollary}

\paragraph{Extension to Supervised Contrastive Learning.}
In unsupervised contrastive learning, positive pairs usually consist of two randomly augmented views, whereas in supervised contrastive learning, positive pairs are drawn i.i.d. from the same class\footnote{For the purpose of simplifying our analysis, we omit positive pairs arising from the augmentations of the same sample.}, embodying the principle of class-conditional independence between positive pairs (i.e., $\Dpos(x,x^+|c)=\gD^c(x)\gD^c(x^+)$, and $\Disc(f)=0$). Consequently, all results about random negative samples presented by \citet{saunshi2019theoretical} are relevant to the present setting.

In the supervised scenario, negative samples can instead be drawn from classes other than that of the anchor, that is, $\Dneg(x^-|x)=\gD^{\not y_{x}}(x^-)$. Then supervised contrastive learning aims at minimizing
\begin{equation}
    \label{eq:supervised-contrastive-loss}
\begin{aligned}
&\Lconsup(f,\gD;\ell)=\E_{\substack{x,x^+\sim \Dpos\\\{x_i^-\}_{i=1}^k\iidsim \Dneg}}\ell\left(f,x,x^+,\{x_i^-\}_{i=1}^k\right)\\
=&\E_{c\sim\rT}\E_{\substack{x,x^+\iidsim \gD^c\\\{x_i^-\}_{i=1}^k\iidsim \gD^{\not c}}}\ell\left(f,x,x^+,\{x_i^-\}_{i=1}^k\right),
\end{aligned}
\end{equation}
and its empirical counterpart is $\widehat{L}_{\mathrm{con}}^{\mathrm{sup}}(f;\ell)$. By an argument analogous to that of Theorem \ref{thm:generalization-bound-for-contrastive-learning}, we obtain the following result:

\begin{theorem}
\label{thm:bound-supervised-CL}
Under the loss conditions in Theorem \ref{thm:generalization-bound-for-contrastive-learning}, with probability at least $1-\delta$, for all $f\in\gF$, there exist a linear classifier $W\in\sR^{|\gT|\times d}$ and a supervised loss $h:\sR^{|\gT|}\times\gT\rightarrow \sR$ such that
\begin{equation}
\Lsup(\gD_{\gT},W\widehat{f},h)\le \Lconsup(f,\gD;\ell)+\Gen_M,
\end{equation}
where $\widehat{f}\in\argmin_{f\in\gF}\widehat{L}_{\mathrm{con}}^{\mathrm{sup}}(f;\ell)$, and the incurred supervised loss is formulated as $h=\E_{\{c_i\}_{i=1}^k\iidsim\rho_{\gT}^{\not y}}\ell(f(x),w_y,\{w_{c_i}\}_{i=1}^k)$.
\end{theorem}

\subsection{The Relaxed K-Means Loss on Representations}
As mentioned above, Theorem \ref{thm:bound-for-contrastive-losses} shows that the population supervised loss is bounded by the contrastive loss $\Lcon(f)$ and the internal discrepancy $\Disc(f)$. However, the generalization bound in Theorem \ref{thm:generalization-bound-for-contrastive-learning} is limited by $\Disc(\widehat{f})$, which depends on the empirical risk minimizer $\widehat{f}$. Since $\Disc(f)\le\Disc_a(f)+\Disc_i(f)$ and the positive-pair alignment $\Disc_a(f)$ is already encouraged by contrastive learning \citep{wang2020understanding}, we may also consider minimizing the empirical counterpart of the intra-class discrepancy $\Disc_i(f)$:
\begin{equation}
\widehat{\Disc}_i(f)=\frac{1}{N}\sum_{n=1}^N\left\|f(x_n)-\widehat{\mu}_{y_{x_n}}\right\|,
\end{equation}
where $y_{x_n}$ denotes the label of sample $x_n$, $\widehat{\mu}_{c}=\frac{1}{N_c}\sum_{n=1}^Nf(x_n)\mathbb{I}(y_{x_n}=c)$, and $N_c$ is the number of samples within class $c$.

Combining Theorem \ref{thm:bound-for-contrastive-losses} with a uniform bound on the estimation error of $\widehat{\Disc}_i$ over $\gF$, we can derive the following proposition:
\begin{proposition}
\label{prop:regularized-bound}
Under the notation and conditions in Theorem \ref{thm:generalization-bound-for-contrastive-learning}, with probability at least $1-\delta$, any $\widehat{f}\in\argmin_{f\in\gF}\widehat{L}_{\mathrm{con}}(f)+\eta\cdot\widehat{\Disc}_i(f)$ satisfies
\begin{equation*}
\Lsup(\gD_{\gT},W\widehat{f},h)\le\inf_{f\in\gF}\left\{\Lcon(f)+\eta\cdot\Disc_i(f)\right\}+\eta\cdot\Disc_a(\widehat{f})+\Gen_M+\eta\cdot\Gen_N,
\end{equation*}
where $h$ is the incurred supervised loss in Theorem \ref{thm:bound-for-contrastive-losses}, $\Gen_N=O\left(\frac{\gR_{\gS'}(\gF)}{N}+R\sqrt{\frac{d|\gT|+\log(2/\delta)}{N}}\right)$, and $\gR_{\gS'}(\gF)$ is the (unnormalized) empirical Rademacher complexity of $\gF$ on $\gS'=\{x_n\}_{n=1}^N$.
\end{proposition}

% Compared with Theorem \ref{thm:generalization-bound-for-contrastive-learning}, Proposition \ref{prop:regularized-bound} requires no additional assumptions and replaces $\Disc_i(\widehat{f})$ in the upper bound of $\Disc(\widehat{f})$ with $\Disc_i(f)$ of the comparator $f$, at the cost of the estimation error $\Gen_N$ of $\widehat{\Disc}_i$, while the remaining alignment term $\Disc_a(\widehat{f})$ requires no labels.

\begin{figure}[t]
    \centering
    \subfigure[CIFAR-10]{
    \label{fig:cifar10-cl}
    \includegraphics[width=1.8in]{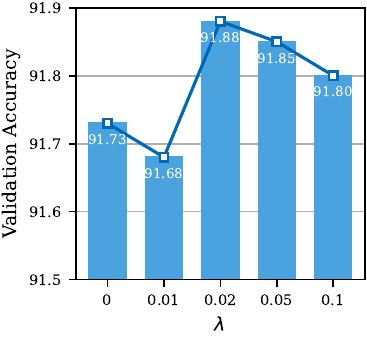}
    }
    \subfigure[CIFAR-100]{
    \label{fig:cifar100-cl}
    \includegraphics[width=1.8in]{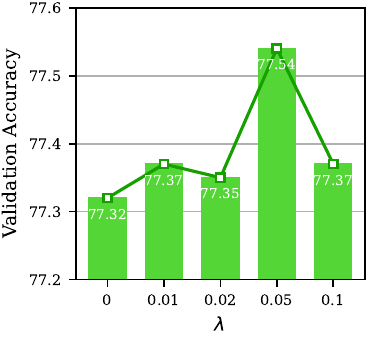}
    }
    \subfigure[ImageNet-100]{
    \label{fig:imagenet100-cl}
    \includegraphics[width=1.8in]{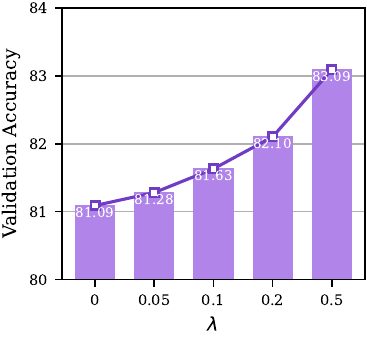}
    }

    \caption{Linear probing accuracies (\%) on CIFAR-10, CIFAR-100, and ImageNet-100 using unsupervised contrastive learning with the InfoNCE loss and the relaxed k-means loss, i.e., $\widehat{L}_{\mathrm{con}}+\lambda\cdot \widehat{L}_{\mathrm{k\text{-}means}}(f)$, where $\tau=0.5$.}
    \label{fig:relaxed-kmeans-loss}

\end{figure}

\paragraph{The Relaxed K-Means Loss.} Due to the absence of specific labels for calculating $\widehat{\Disc}_i(f)$, we resort to using a surrogate. One intuitive approach is to employ k-means clustering to obtain $C$ centroids $\{\hat{\mu}_c\}_{c\in[C]}$ from the representations $\{f(x_i)\}_{i=1}^M$, followed by computing the minimal distance between each representation and these centroids, i.e., $\frac{1}{M}\sum_{i=1}^M\min_{c\in[C]}\|f(x_i)-\hat{\mu}_c\|$. However, this loss is not only inefficient but also requires a predetermined number of clusters. Moreover, the uniformity property of contrastive learning and stochastic batch training would make clustering challenging and potentially imprecise. To alleviate this issue, we consider reducing the distance between representations within a certain distance range. That is, we treat each representation as a center and reduce the distance from neighboring representations to it. This can be straightforwardly implemented with a truncation as follows:
\begin{equation}
\label{eq:kmean-loss}
\widehat{L}_{\mathrm{k\text{-}means}}(f)=\sum_{i,j}\min\left\{\left\|f(x_i)-f(x_j)\right\|,\tau\right\},
\end{equation}
where $\tau>0$ is a threshold to filter overly large pairwise distances. Although it requires no labels, the relaxed k-means loss controls the intra-class discrepancy as follows:
\begin{proposition}
\label{prop:relaxed-kmeans}
Suppose that $\widehat{L}_{\mathrm{k\text{-}means}}(f)$ is computed on the sample $\{x_n\}_{n=1}^N$ in $\widehat{\Disc}_i(f)$, i.e., $i,j\in[N]$ in Equation \ref{eq:kmean-loss}, and that $\tau\le2R$. Write $N_{\min}=\min_{c\in\gT}N_c$ and $N_{\max}=\max_{c\in\gT}N_c$. Then for all $f\in\gF$,
\begin{equation}
    \frac{\tau N_{\min}N}{2R}\widehat{\Disc}_i(f)\le\widehat{L}_{\mathrm{k\text{-}means}}(f)\le2N_{\max}N\widehat{\Disc}_i(f)+\tau\Bigl(N^2-\sum_{c\in\gT}N_c^2\Bigr).
\end{equation}
If, in addition, $\|f(x_i)-f(x_j)\|\ge\tau$ for all $i,j\in[N]$ with $y_{x_i}\neq y_{x_j}$, then
\begin{equation}
    \frac{\tau N_{\min}N}{2R}\widehat{\Disc}_i(f)\le\widehat{L}_{\mathrm{k\text{-}means}}(f)-\tau\Bigl(N^2-\sum_{c\in\gT}N_c^2\Bigr)\le2N_{\max}N\widehat{\Disc}_i(f).
\end{equation}
\end{proposition}
% Proposition \ref{prop:relaxed-kmeans} shows that minimizing the relaxed k-means loss reduces an upper bound on $\widehat{\Disc}_i(f)$ without labels. If pairs from different classes are at least $\tau$ apart, the loss matches $\widehat{\Disc}_i(f)$ up to constant factors and an additive term independent of $f$; otherwise, such pairs within distance $\tau$ also lower the loss as they move closer.
As shown in Figure \ref{fig:relaxed-kmeans-loss}, the incorporation of relaxed k-means loss contributes to an improvement in linear probing, especially on the ImageNet-100 dataset (the complete experimental results can be found in Appendix~\ref{sec:more-results-relaxed-kmeans}). This underscores the feasibility of minimizing the internal discrepancy in this manner. We hope that the relaxed k-means loss can stimulate new insights into unsupervised representation learning.

\section{On the Supervised Losses Incurred by Contrastive Learning}
\label{sec:incurred-losses}
In Section \ref{sec:contrastive-generalization-bound}, contrastive learning-incurred supervised losses are central to analyzing the downstream performance of contrastive learning. These losses entail an expectation over $k$ i.i.d. class draws $\{c_i\}_{i=1}^k$ and differ in form from typical classification losses. In this section, we investigate the behavior of the supervised losses incurred by contrastive learning, focusing specifically on InfoNCE and SpeCL.

\subsection{Implications for Mitigating Class Imbalance}
Real-world data often exhibits an imbalanced class distribution, primarily because some classes have an abundance of examples, while many others are associated with only a few. This class imbalance can make naive learning biased towards the majority classes, leading to poor performance and a narrow feature space for minority classes \citep{ren2020balanced}. However, several studies have underlined the potential of contrastive learning to improve imbalanced learning. For instance, \citet{kang2021exploring} observed that contrastive learning helps learn a balanced feature space, even in the presence of a highly skewed class distribution. Similarly, \citet{li2022targeted} argued that supervised contrastive learning can offer benefits to imbalanced learning. Nevertheless, neither study fully elucidates the underlying mechanisms driving the benefits of contrastive learning for imbalanced learning. In what follows, we theoretically demonstrate that \textit{contrastive learning implicitly mitigates class imbalance}.

\begin{figure}[t]
    \centering
    \subfigure[CIFAR-10-LT]{
    \label{fig:cifar10-lt}
    \includegraphics[width=1.8in]{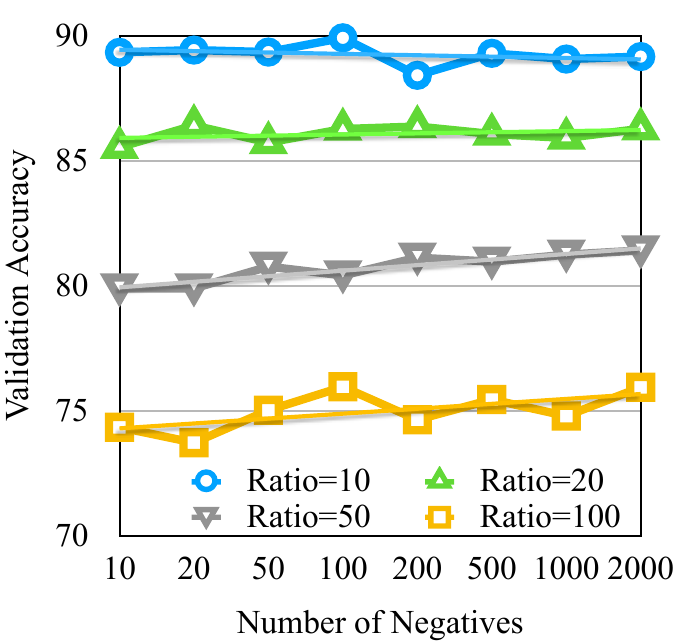}
    }
    \subfigure[CIFAR-100-LT]{
    \label{fig:cifar100-lt}
    \includegraphics[width=1.8in]{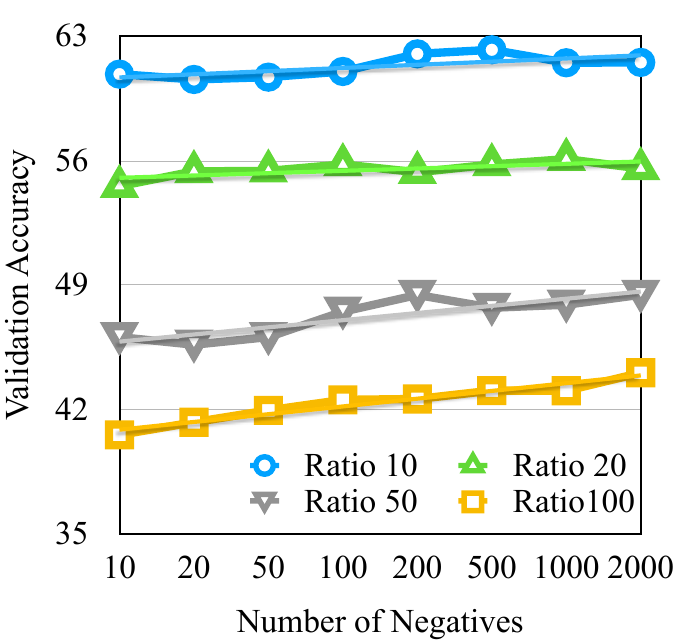}
    }

    \caption{Validation accuracies (\%) on long-tailed CIFAR-10 and long-tailed CIFAR-100 datasets using supervised losses incurred by InfoNCE with different numbers of negatives and imbalance ratios $\{10,20,50,100\}$, where the imbalance ratio denotes the sample size ratio between the largest and smallest classes.}
    \label{fig:long-tailed-cifar}

\end{figure}

\begin{table}[!t]
\centering
\caption{Validation accuracies (\%) on long-tailed CIFAR-10 and long-tailed CIFAR-100 datasets. These results are obtained through supervised learning using cross entropy (CE), balanced softmax (BS), the InfoNCE-incurred loss (with the number of negatives $k$ and the number of sampling repetitions $t=20$), and the SpeCL-incurred (\textsc{SpeCL-i}) loss, where \textsc{SpeCL-i}$^\dagger$ denotes the results with normalized features and classifier weights. It is evident that the supervised loss incurred by InfoNCE outperforms CE, and \textsc{SpeCL-i}$^\dagger$ demonstrates superior performance compared to \textsc{SpeCL-i}. Furthermore, the performance of the InfoNCE-incurred loss generally improves as the number of negatives increases, particularly in scenarios with larger imbalance ratios.}
\label{tab:imbalanced-learning}

\begin{sc}
\begin{tabular}{c|c|cccc}
\toprule
\multirow{2}{*}{\textbf{Dataset}}
& \multirow{2}{*}{\textbf{Method}}
& \multicolumn{4}{c}{\textbf{Imbalance Ratio}} \\
& & 10 & 20 & 50 & 100 \\
\midrule
\midrule
\multirow{8}{*}{\textbf{CIFAR-10-LT}}
& CE                & 87.85 & 83.26 & 75.12 & 66.19 \\
& BS                & 89.33 & 86.22 & 81.99 & 76.12 \\
\cmidrule{2-6}
& $k=10$            & 89.35 & 85.55 & 79.90 & 74.34 \\
& $k=20$            & 89.44 & 86.42 & 79.90 & 73.72 \\
% & $k=50$            & 89.36 & 85.74 & 80.77 & 75.02 \\
% & $k=100$           & 89.93 & 86.28 & 80.42 & 75.96 \\
& $k=200$           & 88.43 & 86.38 & 81.13 & 74.64 \\
% & $k=500$           & 89.31 & 86.08 & 80.97 & 75.45 \\
% & $k=1000$          & 89.08 & 85.92 & 81.25 & 74.79 \\
& $k=2000$          & 89.18 & 86.29 & 81.42 & 75.94 \\
& SpeCL-i          & 88.14 & 84.12 & 76.20 & 69.62 \\
& SpeCL-i$^\dagger$ & 88.68 & 85.26 & 78.41 & 72.95 \\
\midrule
\multirow{8}{*}{\textbf{CIFAR-100-LT}}
& CE                & 59.62 & 52.71 & 43.65 & 37.69 \\
& BS                & 62.30 & 55.80 & 48.91 & 43.80 \\
\cmidrule{2-6}
& $k=10$            & 60.86 & 54.56 & 46.07 & 40.57 \\
% & $k=50$            & 60.69 & 55.42 & 46.07 & 41.95 \\
& $k=20$            & 60.56 & 55.39 & 45.64 & 41.28 \\
% & $k=100$           & 61.01 & 55.78 & 47.50 & 42.54 \\
& $k=200$           & 62.00 & 55.34 & 48.44 & 42.59 \\
% & $k=500$           & 62.23 & 55.78 & 47.72 & 43.06 \\
% & $k=1000$          & 61.50 & 56.08 & 47.89 & 43.05 \\
& $k=2000$          & 61.52 & 55.55 & 48.45 & 44.06 \\
& SpeCL-i           & 59.34 & 51.27 & 43.99 & 38.50 \\
& SpeCL-i$^\dagger$ & 61.03 & 53.61 & 46.51 & 41.11 \\
\bottomrule
\end{tabular}
\end{sc}
\end{table}

\paragraph{InfoNCE-Incurred Loss Approximates Balanced Softmax Loss.}
We demonstrate that the loss incurred by InfoNCE closely approximates the balanced softmax loss \citep{ren2020balanced}, which accommodates class distribution shifts to mitigate class imbalance. The balanced softmax loss is expressed as $\tilde{h}=\log\left[\E_{c\sim\rT}e^{f(x)^\top(w_c-w_y)}\right]$, where we omit the term $-\log\rT(y)$ that is independent of $f$ and $W$. We also note that the InfoNCE-incurred loss $h$ in Theorem \ref{thm:generalization-bound-for-contrastive-learning} is defined as $h(x,y)=\E_{\{c_i\}_{i=1}^k\iidsim\rho_{\gT}}\log\left[1+\sum_{i=1}^k e^{f(x)^\top(w_{c_i}-w_{y})}\right]$, which can be interpreted as an approximation of $\log\left[1+k\E_{c\sim\rT}e^{f(x)^\top(w_c-w_y)}\right]$ that resembles $\tilde{h}$. Theoretically, we confirm that

\begin{theorem}
\label{thm:approximation-infonce-balanced-softmax}
For the InfoNCE-incurred loss $h$ and the balanced softmax loss $\tilde{h}$, we have $h=\tilde{h}+\log k+O(k^{-1})$, and with probability at least $1-\delta$, for all $f\in\gF$,  $\Lsup(\gD_{\gT},W\widehat{f},\tilde{h})+\log k\le \Lcon(f)+\eta\cdot\Disc(\widehat{f})+\Gen_{M}+O(k^{-1})$.
\end{theorem}
The theorem above demonstrates that the loss incurred by InfoNCE approximates the balanced softmax loss, with an error term $|h-\tilde{h}-\log k|$ that converges to zero at a rate of $O(k^{-1})$. In other words, this approximation becomes progressively more accurate as the number of negatives $k$ increases. As shown in Figure \ref{fig:long-tailed-cifar} and Table \ref{tab:imbalanced-learning}, the performance of the InfoNCE-incurred loss usually improves for highly skewed class distributions (imbalance ratios of 50 and 100) as $k$ increases, which suggests that contrastive learning with more negative samples can mitigate class imbalance.

% For fixed representations and class-conditional distributions, with a uniform test prior, the conditional minimizers of balanced softmax over unrestricted scores $s_c$ give the balanced-error decision rule $\argmax_c s_c$; under the original class prior, the corresponding rule is $\argmax_c\{s_c+\log\rT(c)\}$ (Appendix \ref{sec:balanced-risk-interpretation}). This interpretation does not automatically extend to the constrained mean classifier.

\paragraph{Class Imbalance Only Scales Representations by SpeCL-Incurred Loss.}
For the SpeCL-incurred supervised loss $h(x,y)=-2f(x)^\top w_y+\E_{c\sim\rT}(f(x)^\top w_c)^2$, we focus on the representations that minimize the population risk $\E_{(x,y)\sim\gD_{\gT}} h(x,y)$. We demonstrate that an imbalanced class distribution only scales the norms of corresponding optimal representations and classifier weights by factors $\frac{s}{\sqrt{\rT(c)|\gT|}}$ and $\frac{1}{s\sqrt{\rT(c)|\gT|}}$, respectively (see Appendix \ref{sec:val-specl-imbalance}), while preserving their angular geometry. To elucidate this, we prove that
\begin{theorem}
\label{thm:balanced-SpeCL}
For the risk $\E_{(x,y)\sim\gD_{\gT}} h(x,y)$, minimized over all representations and classifier weights without constraints, let $\tilde{f}(x)$ and $\tilde{w}_c$ jointly be a global minimizer under the condition that $\rT$ is uniform. Then for any $\rT$ and any constant $s>0$, the representations $\hat{f}(x)=\frac{s}{\sqrt{\rT(c)|\gT|}} \tilde{f}(x)$ and classifier weights $\hat{w}_c=\frac{1}{s\sqrt{\rT(c)|\gT|}}\tilde{w}_c$, defined for all $x\sim\gD^c$ and $c\in\gT$, form a global minimizer of $\E_{(x,y)\sim\gD_{\gT}} h(x,y)$, where we further have $\hat{f}(x)^\top\hat{w}_c=\frac{1}{\rT(c)|\gT|}\tilde{f}(x)^\top\tilde{w}_c$ for $x\sim\gD^c$. Conversely, for any $s>0$, every global minimizer under $\rT$ arises in this way from a global minimizer under the uniform $\rT$.
\end{theorem}

Since the risk depends on $f$ and $W$ only through the logits $f(x)^\top w_c$, it is invariant under $f(x)\mapsto Af(x)$ and $w_c\mapsto A^{-\top}w_c$ for any invertible $A\in\sR^{d\times d}$, so its global minimizers do not have a unique geometry. Normalizing features and classifier weights allows only orthogonal $A$, which preserve the geometry, and we further prove that

% Since the risk depends only on the logits $f(x)^\top w_c$, it is invariant under $f(x)\mapsto Af(x)$ and $w_c\mapsto A^{-\top}w_c$ for any invertible $A\in\sR^{d\times d}$, so its global minimizers have no unique geometry. Constraining the norms of features and classifier weights removes this ambiguity:

\begin{proposition}
\label{prop:normalized-SpeCL}
For the risk $\E_{(x,y)\sim\gD_{\gT}} h(x,y)$, minimized over all representations and classifier weights satisfying $\|f(x)\|\le R$ and $\|w_c\|\le R$, assume that $d\ge|\gT|$ and $R^2\le1/\max_{c\in\gT}\rT(c)$. Then $f$ and $\{w_c\}_{c\in\gT}$ jointly form a global minimizer if and only if $\{w_c/R\}_{c\in\gT}$ is orthonormal and $f(x)=w_{y_x}$ almost surely, and the minimum is $-2R^2+R^4\sum_{c\in\gT}\rT^2(c)$. The same holds under the normalization $\|f(x)\|=\|w_c\|=R$.
\end{proposition}

\paragraph{Remark.} Theorem \ref{thm:balanced-SpeCL} suggests that cosine similarities between features or classifier weights optimally learned by SpeCL-incurred losses remain consistent regardless of the class distribution. Hence, after projection onto the unit sphere $\mathbb{S}^{d-1}$, the nonzero optimal features and classifier weights share the same geometry as their counterparts in the class-balanced case. Moreover, \textit{we can obtain a rebalanced feature space and make balanced decisions by simply normalizing features and classifier weights within SpeCL}, as guaranteed by Proposition \ref{prop:normalized-SpeCL} and demonstrated in Table \ref{tab:imbalanced-learning}.

\subsection{The Optimality of Contrastive Learning-Incurred Losses}
The optimal solutions of the CE loss, the MSE loss, and other losses usually exhibit neural collapse, an intriguing empirical phenomenon \citep{papyan2020prevalence, mixon2020neural, fang2021exploring, han2022neural, zhou2023dynamics, wu2025neural}, where specific manifestations in the classifier and last-layer representations are characterized. \citet{graf2021dissecting, jiang2024supervised} also established neural collapse results in the empirical setting of supervised contrastive learning.
Our objective here is to analyze the optimality of unsupervised or supervised contrastive learning by analyzing their incurred supervised losses.

% These characterizations do not by themselves establish optimality for the corresponding contrastive objectives.

\begin{theorem}
\label{thm:optimality-of-InfoNCE}
For the class-balanced distribution and the InfoNCE-incurred supervised losses $\Lsup=\E_{x\sim\gD}\E_{\{c_i\}_{i=1}^k\iidsim\rT}\ell(f(x),\mu_{y_x},\{\mu_{c_i}\}_{i=1}^k)$ in Theorem \ref{thm:bound-for-contrastive-losses} and $\Lsup'=\E_{x\sim\gD}\E_{\{c_i\}_{i=1}^k\iidsim\rT^{\not y_x}}\ell(f(x),\mu_{y_x},\{\mu_{c_i}\}_{i=1}^k)$ in Theorem \ref{thm:bound-supervised-CL} (where $\mu_{c}=\E_{x\sim\gD^c}f(x)$ for $c\in\gT$), if $d\ge |\gT|-1$, then $f$ minimizes $\Lsup$ or $\Lsup'$ over all representations satisfying $\|f(x)\|\le R$ if and only if $f(x)=\mu_{y_x}$ almost surely, where $\mu_{c}^\top\mu_{c'}=-\frac{R^2}{|\gT|-1}$ for $c\neq c'$ and $\|\mu_{c}\|=R$.
\end{theorem}

Theorem \ref{thm:optimality-of-InfoNCE} shows that the InfoNCE-incurred loss possesses an optimal structure that aligns with the structures demonstrated for CE \citep{papyan2020prevalence} and MSE \citep{zhou2022optimization}. Specifically, these representations exhibit two critical properties of neural collapse \citep{papyan2020prevalence}: (i) \textit{variability collapse}, where the representations equal their class means; (ii) \textit{a simplex equiangular tight frame (ETF)}, wherein the class means have equal length, form equal-sized angles between any given pair, and attain the maximal minimum pairwise distance.

\begin{corollary}
\label{cor:near-optimal-infonce}
Under the conditions of Theorem \ref{thm:optimality-of-InfoNCE}, let $\bar{\mu}=|\gT|^{-1}\sum_{c\in\gT}\mu_c$, let $G=(\mu_c^\top\mu_{c'})_{c,c'\in\gT}$ be the Gram matrix of the class means, and let $G_{\mathrm{ETF}}=\frac{|\gT|R^2}{|\gT|-1}(\mI-\frac{1}{|\gT|}\vone\vone^\top)$ be that of the simplex ETF in Theorem \ref{thm:optimality-of-InfoNCE}.

(i) If $\Lsup$ or $\Lsup'$ is within $\eps\ge0$ of its global minimum, then
\begin{equation}
    \label{eq:near-optimal-infonce}
    \begin{aligned}
        \frac{1}{|\gT|}\sum_{c\in\gT}\E_{x\sim\gD^c}\|f(x)-\mu_c\|^2+\|\bar{\mu}\|^2&\le2e^{2R^2}\eps,\\
        \frac{1}{|\gT|^2}\|G-G_{\mathrm{ETF}}\|_F^2&\le8e^{4R^2}\eps.
    \end{aligned}
\end{equation}
% Explicit constants and the corresponding stability bound for the Gram matrix of the class means are given in Appendix \ref{sec:proof-near-optimal-infonce}.

(ii) For the empirical minimizer $\widehat{f}$ in Theorem \ref{thm:generalization-bound-for-contrastive-learning}, assume additionally that $\gF$ contains a class-constant simplex ETF of radius $R$ and that $\sP_{x,x^+\sim\Dpos}[y_x\ne y_{x^+}]\le\alpha$ with $x^+$ following the marginal distribution $\gD$. Writing $\Lsup(f)$ for the unprimed mean-classifier risk in Theorem \ref{thm:optimality-of-InfoNCE} and $\Lsup^*$ for its global minimum, with probability at least $1-\delta$,
\begin{equation}
    \label{eq:learned-infonce-excess}
    \Lsup(\widehat{f})-\Lsup^*\le R\Disc(\widehat{f})+\Gen_M+\sqrt{2|\gT|/(|\gT|-1)}R^2\alpha.
\end{equation}
Equation \ref{eq:near-optimal-infonce} therefore holds for $f=\widehat{f}$ with $\eps$ equal to this upper bound.
\end{corollary}

Corollary \ref{cor:near-optimal-infonce} extends Theorem \ref{thm:optimality-of-InfoNCE} in two aspects: (i) any near-optimal representation approximately exhibits variability collapse with nearly centered class means, and the Gram matrix of its class means approximates that of a simplex ETF. (ii) by taking a class-constant simplex ETF in $\gF$ as the comparator in Theorem \ref{thm:generalization-bound-for-contrastive-learning}, we show that the empirical minimizer $\widehat{f}$ is near-optimal, with its excess risk controlled by the internal discrepancy $\Disc(\widehat{f})$, the generalization error $\Gen_M$, and the term $\sqrt{2|\gT|/(|\gT|-1)}R^2\alpha$ accounting for positive pairs from different classes. Consequently, the representations learned by InfoNCE approximately exhibit variability collapse and a simplex ETF when these terms are small.
% Under conditional independence of positive pairs as in Section \ref{sec:internal-discrepancy}(i), we have $\Disc(\widehat{f})=0$ and $\alpha=0$, so that Equation \ref{eq:near-optimal-infonce} holds for $\widehat{f}$ with $\eps=\Gen_M$.

The SpeCL-incurred loss also admits optimal representations exhibiting variability collapse. Its optimal class means minimize the error in a constrained approximation of the identity matrix $\mI$, which becomes a best rank-$d$ approximation under the condition in Equation \ref{eq:specl-minimum-attainability}:
\begin{theorem}
\label{thm:optimality-of-specl}
For the SpeCL-incurred supervised loss $\Lsup=\E_{(x,y)\sim\gD_{\gT}}\E_{\{c_i\}_{i=1}^k\iidsim\rT}\ell(f(x),\mu_{y},\{\mu_{c_i}\}_{i=1}^k)$ in Theorem \ref{thm:bound-for-contrastive-losses}, where $\ell(f(x),\mu_{y},\{\mu_{c_i}\}_{i=1}^k)=-2f(x)^\top \mu_{y}+\frac{1}{k}\sum_{i=1}^k (f(x)^\top\mu_{c_i})^2$, a global minimum is attained by a representation satisfying $f(x)=\mu_{c}^*$ almost surely under $\gD^c$ for all $c\in\gT$, where $H=[\sqrt{\rT(c)}\mu_{c}^*]_{c\in\gT}$ minimizes $\|\mI-H^\top H\|_F^2$ subject to $\|\mu_c^*\|\le R$. More specifically, writing $r=\min\{d,|\gT|\}$, the minimum of $\Lsup$ equals $-r$ if and only if
\begin{equation}
    \label{eq:specl-minimum-attainability}
    \sum_{c\in\gT}\min\{1,R^2\rT(c)\}\ge r.
\end{equation}
In this case, for every global minimizer $f$, the matrix $H$ built from its class means $\mu_c=\E_{x\sim\gD^c}f(x)$ satisfies that $H^\top H$ is a rank-$r$ orthogonal projection: $H^\top H=\mI$ when $|\gT|\le d$, and $H^\top H$ is a best rank-$d$ approximation of $\mI$ when $|\gT|>d$. For uniform $\rT$, the minimum is
\begin{equation}
    \label{eq:specl-balanced-minimum}
    \inf_f\Lsup=
    \begin{cases}
        -2R^2+R^4/r,& R^2\le r,\\
        -r,& R^2\ge r.
    \end{cases}
\end{equation}
\end{theorem}

When $|\gT|\le d$, the condition in Equation \ref{eq:specl-minimum-attainability} is equivalent to $R^2\ge1/\min_{c\in\gT}\rT(c)$. In this case, $H^\top H=\mI$ means that the optimal class means are mutually orthogonal with $\|\mu_c\|^2=1/\rT(c)$ for any $\rT$, which is the counterpart of Theorem \ref{thm:balanced-SpeCL} for the mean classifier: class imbalance only scales the optimal class means.

% Spectral optimality does not imply class separation. For four balanced classes with $d=2$ and $R=1$, the class-constant representations $\mu_1=\mu_2=e_1$ and $\mu_3=\mu_4=e_2$, where $e_1,e_2$ are standard basis vectors, attain risk $-3/2$, the global minimum in Equation \ref{eq:specl-balanced-minimum}. However, every classifier based only on these representations has accuracy at most $1/2$.

\section{Conclusion}
This paper introduces the novel concept of internal discrepancy as a unifying measure for various assumptions. This measure effectively captures the relationships between representations of positive pairs and latent clusters. By establishing an internal discrepancy-aware generalization bound, we provide a bridge for understanding contrastive learning without requiring additional assumptions, while also yielding practical strategies like the relaxed k-means loss. Furthermore, we delve into the behaviors of contrastive learning-incurred supervised losses. We show that under class imbalance, the InfoNCE-incurred loss approximates the balanced softmax loss, while features learned by the SpeCL-incurred loss maintain consistent cosine similarities regardless of class distributions. This helps explain why contrastive learning can mitigate class imbalance. The optimality of contrastive learning-incurred losses also involves some intriguing properties, including neural collapse properties for InfoNCE and the best approximation to the identity matrix for SpeCL. These results not only enhance our understanding of contrastive learning but also pave the way for exploring downstream behaviors.

% \blindmathpaper

% Here is a citation \cite{chow:68}.

% % Acknowledgements and Disclosure of Funding should go at the end, before appendices and references

\acks{This work is supported by National Natural Science Foundation of China under Grant 62525107.}


\newpage
\appendix

\begin{center}
    \Large\textbf{Appendix: Internal Discrepancy for Contrastive Learning}
\end{center}

\section{Proofs of Theorems, Lemmas and Propositions}
\label{sec:proofs}

\subsection{The Convexity and Lipschitzness of the InfoNCE Loss and Spectral Contrastive Loss}
\label{appendix:convexity-lipschitzness}

For fixed remaining representations, the InfoNCE loss is a log-sum-exp of affine functions, and the spectral contrastive loss is a linear function plus nonnegative quadratic terms. Thus both losses are separately convex in each representation. Their Lipschitz constants with respect to the positive representation can be taken as $\eta=R$ and $\eta=2R$, respectively, by differentiating with respect to that representation.

\begin{lemma}
\label{lemma-lipschitzness-of-margin}
The margin function $\phi:\sR^{(k+2)d}\rightarrow \sR^k$ with $\phi(v,v^+,v_1^-,\ldots,v_k^-)=\left(v^\top(v^+-v_i^-)\right)_{i\in[k]}$ is $\sqrt{6k}R$-Lipschitz continuous when $\|v\|_2,\|v^+\|_2,\|v_i^-\|_2\le R$.
\end{lemma}
\begin{proof}
    The nonzero blocks in row $i$ of the Jacobian $J$ are $v^+-v_i^-$, $v$, and $-v$, corresponding to $v$, $v^+$, and $v_i^-$. Hence
    \begin{equation*}
        \|J\|_2^2\le\|J\|_F^2
        =\sum_{i=1}^k\bigl(\|v^+-v_i^-\|_2^2+2\|v\|_2^2\bigr)
        \le6kR^2.
    \end{equation*}
    Integrating the Jacobian along line segments in the convex product of norm balls proves the claim.
\end{proof}

\begin{lemma}
\label{lemma-lipschitzness-of-spectral}
The spectral contrastive function $\phi:\sR^{(k+2)d}\rightarrow\sR^k$ with $\phi(v,v^+,v_1^-,\ldots,v_k^-)=\left(2v^\top v^+-(v^\top v_i^-)^2\right)_{i\in[k]}$ is $\sqrt{12kR^2+12kR^6}$-Lipschitz continuous when $\|v\|_2,\|v^+\|_2,\|v_i^-\|_2\le R$.
\end{lemma}
\begin{proof}
    To show that $\phi$ is Lipschitz continuous, we only need to prove that its Jacobian is bounded. Indeed, for all $i\in[k]$,
    \begin{equation*}
        \frac{\partial\phi_i}{\partial v}=2v^+-2v_i^-(v^\top v_i^-),\qquad
        \frac{\partial\phi_i}{\partial v^+}=2v,\qquad
        \frac{\partial\phi_i}{\partial v_j^-}=-2v(v^\top v_j^-)\mathbb{I}(i=j).
    \end{equation*}
    The Frobenius norm of the Jacobian $J$ therefore satisfies
    \begin{equation*}
        \begin{aligned}
            \|J\|_F^2
            &=\sum_{i=1}^k\left(\|2v^+-2v_i^-(v^\top v_i^-)\|_2^2
                +4\|v\|_2^2+4\|v(v^\top v_i^-)\|_2^2\right)\\
            &\le12kR^2+12kR^6.
        \end{aligned}
    \end{equation*}
    Since $\|J\|_2\le\|J\|_F$ and the domain is convex, the stated Lipschitz bound follows.
\end{proof}

Direct differentiation gives sharper tuple Lipschitz constants. For InfoNCE, write $q_i=e^{v^\top(v_i^--v^+)}/(1+\sum_j e^{v^\top(v_j^--v^+)})$ and $s=\sum_iq_i\le1$. Then
\begin{equation*}
    \nabla_v\ell=\sum_iq_i(v_i^--v^+),\qquad
    \nabla_{v^+}\ell=-sv,\qquad
    \nabla_{v_i^-}\ell=q_iv,
\end{equation*}
so $\|\nabla\ell\|^2\le R^2(5s^2+\sum_iq_i^2)\le6R^2$ and $L\le\sqrt{6}R$. Its values lie in $[\log(1+ke^{-2R^2}),\log(1+ke^{2R^2})]$, an interval of length at most $4R^2$. For SpeCL, differentiation gives
\begin{equation*}
    \nabla_v\ell=-2v^++\frac{2}{k}\sum_i(v^\top v_i^-)v_i^-,\qquad
    \nabla_{v^+}\ell=-2v,\qquad
    \nabla_{v_i^-}\ell=\frac{2}{k}(v^\top v_i^-)v,
\end{equation*}
and hence $\|\nabla\ell\|^2\le4[(R+R^3)^2+R^2+R^6/k]\le8(R+R^3)^2$. Thus $L\le2\sqrt{2}R(1+R^2)$, and its values lie in $[-2R^2,2R^2+R^4]$. The conditions of Theorem \ref{thm:generalization-bound-for-contrastive-learning} hold with $(\eta,\zeta_1(R),\zeta_2(R))=(R,\sqrt{6},4R^2)$ for InfoNCE and $(2R,\sqrt{2}(1+R^2),4R^2+R^4)$ for SpeCL.

\subsection{Generalization Bounds for Contrastive Learning}

\subsubsection{Proof of Theorem \ref{thm:bound-for-contrastive-losses}}

\begin{proof}
The key idea in the proof is the use of the Lipschitzness of the loss function $\ell$ and Jensen's inequality based on the convexity of $\ell$ with respect to $f(x^+)$ and $f(x_i^-)$. The unsupervised loss uses a random point from a class as a classifier; using the mean of the class as the classifier should only make the loss lower. Let $\mu_c=\E_{x\sim\gD^c} f(x)$ be the mean representation of the class $c\in\gT$ and $\mu_{y_x}$ denote the mean representation of the class of the sample $x$. Then we have
\begin{equation*}
\begin{aligned}
\Lcon(f,\gD)=&\E_{\substack{(x,x^+)\sim \Dpos\\\{x_i^-\}_{i=1}^k\iidsim \gD}}\ell\left(f(x),f(x^+),\{f(x_i^-)\}_{i=1}^k\right)\\
=&\E_{x\sim\gD}\E_{x^+\sim\Dpos(\cdot|x)}\E_{\{x_i^-\}_{i=1}^k\iidsim \gD}\ell\left(f(x),f(x^+),\{f(x_i^-)\}_{i=1}^k\right)\\
\ge &\E_{x\sim\gD}\E_{\{x_i^-\}_{i=1}^k\iidsim \gD}\ell\left(f(x),\E_{x^+\sim\Dpos(\cdot|x)} f(x^+),\{f(x_i^-)\}_{i=1}^k\right)\\
\ge &\E_{x\sim\gD}\E_{\{c_i\}_{i=1}^k\iidsim\rho_{\gT}}\ell\left(f(x),\mu_{y_x}+\E_{x^+\sim\Dpos(\cdot|x)} f(x^+)-\mu_{y_x},\{\mu_{c_i}\}_{i=1}^k\right)\\
\ge & \E_{x\sim\gD}\E_{\{c_i\}_{i=1}^k\iidsim\rho_{\gT}}\ell\left(f(x),\mu_{y_x}, \{\mu_{c_i}\}_{i=1}^k\right)\\
&\quad-\eta\cdot\E_{x\sim\gD}\left\|\E_{x^+\sim\Dpos(\cdot|x)} f(x^+)-\mu_{y_x}\right\|\\
=& \E_{x\sim\gD}h\left(\mu_{y_x}^\top f(x),\{\mu_{c}^\top f(x)\}_{c\not=y_x,c\in\gT}\right) -\eta\cdot\Disc(f),
% =&\E_{x\sim\gD}\E_{\{c_i\}_{i=1}^k\iidsim\rho_{\gT}}\ell\left(f(x),\mu_{y_x}, \{\mu_{c_i}\}_{i=1}^k\right)-\eta\cdot\Disc(f).
\end{aligned}
\end{equation*}
where $W=(\mu_c)_{c\in\gT}^\top$ and $h\left(\mu_{y_x}^\top f(x),\{\mu_{c}^\top f(x)\}_{c\not=y_x,c\in\gT}\right)=\E_{\{c_i\}_{i=1}^k\iidsim\rho_{\gT}}\ell\left(f(x),\mu_{y_x}, \{\mu_{c_i}\}_{i=1}^k\right)$. The first and second inequalities follow by applying Jensen's inequality and the convexity of $\ell$ with respect to $f(x^+)$ and $f(x_i^-)$, and the third inequality follows from the Lipschitzness of $\ell$. Under the stated class-constant condition, $\Disc(f)=0$ and all three inequalities are equalities.

\end{proof}

\subsubsection{Proof of Proposition \ref{prop:tightness}}

\begin{proof}
    For any $f'$ satisfying $\|f'(x)\|\le R$, let $\bar{\mu}'=\E_{x\sim\gD}f'(x)$. Since $x^+=x$, the InfoNCE loss is $\log[1+\sum_{i=1}^ke^{z_i}]$ with $z_i=f'(x)^\top f'(x_i^-)-\|f'(x)\|^2$, which is convex in $(z_1,\ldots,z_k)$. Applying Jensen's inequality to the negatives and then to the anchor gives
    \begin{equation*}
        \begin{aligned}
            \Lcon(f')&\ge\E_{x\sim\gD}\log\left[1+ke^{f'(x)^\top\bar{\mu}'-\|f'(x)\|^2}\right]\\
            &\ge\log\left[1+ke^{\|\bar{\mu}'\|^2-\E_{x\sim\gD}\|f'(x)\|^2}\right]
            \ge\log\left(1+ke^{-R^2}\right).
        \end{aligned}
    \end{equation*}

    Let $m=\ceil{k/(2\eps)}$, and let $\{u_{c,i}\}_{c\in\gT,i\in[m]}$ be orthonormal vectors in $\sR^d$, which exist since $d\ge m|\gT|$. Let $\gD$ assign probability $\rT(c)/(2m)$ to each of $2m$ distinct points $x_{c,i}$ and $x_{c,i}'$, $i\in[m]$, of class $c$, and set $f(x_{c,i})=Ru_{c,i}$ and $f(x_{c,i}')=-Ru_{c,i}$. Then $\E_{x\sim\gD^c}f(x)=0$ for every $c\in\gT$, so the mean classifier is $W=0$ and $\Lsup(\gD_{\gT},Wf,h)=\log(1+k)$. Moreover, $x^+=x$ and $\|f(x)\|=R$ give $\Disc(f)=R$. For an anchor $x$ and $z_i$ defined by $f$, we have $z_i=0$ if $x_i^-=x$ and $z_i\le-R^2$ otherwise. Letting $J$ be the number of negatives equal to $x$, we obtain
    \begin{equation*}
        \log\left[1+\sum_{i=1}^ke^{z_i}\right]\le\log\left[1+J+(k-J)e^{-R^2}\right]\le\log\left(1+ke^{-R^2}\right)+J.
    \end{equation*}
    Since $\E J=k\rT(y_x)/(2m)\le\eps$, we have $\Lcon(f)\le\log(1+ke^{-R^2})+\eps\le\inf_{f'}\Lcon(f')+\eps$. Finally, $\Lcon(f)+\eta\cdot\Disc(f)\le\log(1+ke^{-R^2})+R^2+\eps=\log(e^{R^2}+k)+\eps$, which rearranges to the stated inequality.
\end{proof}

\subsubsection{Proof of Lemma \ref{lemma:infonce-ce}}

\begin{proof}
    Write $C=|\gT|$, $\mu_c=\E_{x\sim\gD^c}f(x)$, $\bar{\mu}=C^{-1}\sum_c\mu_c$, and $\Delta=C^{-1}\sum_c\|\mu_c\|^2-\|\bar{\mu}\|^2\le R^2$. Define
    \begin{equation*}
        \Phi(z)=\sum_{j=0}^k\binom{k}{j}C^{-j}(1-C^{-1})^{k-j}
        \log[1+j+(k-j)e^z].
    \end{equation*}
    Conditioning on the number of negative classes equal to the target class, and applying Jensen's inequality to the other negative classes, the anchor, and the target class, gives
    \begin{equation*}
        \gH(f)\ge\Phi\left(-\frac{C\Delta}{C-1}\right)
        \ge\Phi\left(-\frac{CR^2}{C-1}\right).
    \end{equation*}
    Every class-constant representation whose class means form a simplex ETF of radius $R$ attains this bound, and such representations exist since their Gram matrix $\frac{CR^2}{C-1}(\mI-\frac{1}{C}\vone\vone^\top)$ is positive semidefinite of rank $C-1\le d$. Since $\Phi$ is strictly increasing, every global minimizer has $\Delta=R^2$, hence $\bar{\mu}=0$ and $\|\mu_c\|=R$ for every $c$. Moreover, $\E_{x\sim\gD^c}\|f(x)-\mu_c\|^2=\E_{x\sim\gD^c}\|f(x)\|^2-\|\mu_c\|^2\le R^2-\|\mu_c\|^2=0$, so $f(x)=\mu_c$ almost surely. The event with exactly $k-1$ negatives in the target class has positive probability. Equality in its strictly convex scalar Jensen step implies that $\mu_c^\top\mu_{c'}$ is constant over $c'\neq c$. As $\bar{\mu}=0$, this constant is $-R^2/(C-1)$. Hence $f$ minimizes $\gH(f)$ if and only if $f(x)=\mu_{y_x}$ almost surely, where $\|\mu_c\|=R$ and $\mu_c^\top\mu_{c'}=-R^2/(C-1)$ for $c\neq c'$.

    Applying Jensen's inequality in the same way to the CE risk gives
    \begin{equation*}
        \begin{aligned}
            &\E_{(x,y)\sim\gD_\gT}\log\left[1+\sum_{c\neq y}e^{f(x)^\top(\mu_c-\mu_y)}\right]\\
            &\qquad\ge\log\left[1+(C-1)e^{-C\Delta/(C-1)}\right]
            \ge\log\left[1+(C-1)e^{-CR^2/(C-1)}\right].
        \end{aligned}
    \end{equation*}
    Every minimizer characterized above attains this bound over the same mean-classifier representations. Conversely, equality in this chain requires $\Delta=R^2$, which again implies $\bar{\mu}=0$, $\|\mu_c\|=R$, and $f(x)=\mu_c$ almost surely. Equality in the strictly convex Jensen step over $c\neq y$ then implies that $\mu_y^\top\mu_c$ is constant over $c\neq y$, so every minimizer of the CE risk is also a minimizer of $\gH(f)$, completing the proof.
\end{proof}

\subsubsection{Proof of Theorem \ref{thm:generalization-bound-for-contrastive-learning}}

\begin{proof}
Let $w_c=\E_{x\sim\gD^c}\widehat{f}(x)$ and let $W$ have rows $w_c^\top$. Then $\Lsup(\gD_{\gT},W\widehat{f},h)=\E_{x\sim\gD}h(w_{y_x}^\top\widehat{f}(x),\{w_c^\top\widehat{f}(x)\}_{c\neq y_x,c\in\gT})$. According to Theorem \ref{thm:bound-for-contrastive-losses}, we have
\begin{equation*}
    \Lsup(\gD_{\gT}, W\widehat{f},h)\le \Lcon(\widehat{f})+\eta\cdot\Disc(\widehat{f}).
\end{equation*}
We further apply Theorem \ref{theorem:generalization-bound-rademacher}. Then, with probability at least $1-\delta$, for all $f\in\gF$,
\begin{equation*}
    \Lsup(\gD_{\gT}, W\widehat{f},h)\le  \Lcon(f)+\eta\cdot\Disc(\widehat{f})+O\left(\frac{\eta\zeta_1(R)\gR_\gS(\gF)}{M}+\zeta_2(R)\sqrt{\frac{\log(2/\delta)}{M}}\right),
\end{equation*}
which finishes the proof.
\end{proof}

The vector-contraction inequality is as follows:
\begin{lemma}[Corollary 4 in \citet{maurer2016vector}]
\label{lipschitz-rademacher}
Let $Z$ be any set and $\gS=\{z_j\}_{j=1}^M\in Z^M$. Let $\tilde{\gF}$ be a class of functions $\tilde{f}:Z\rightarrow \sR^n$ and $h:\sR^n\rightarrow \sR$ be $L$-Lipschitz. For all $\tilde{f}\in\tilde{\gF}$, let $g_{\tilde{f}}=h\circ \tilde{f}$. Then
\begin{equation}
    \E_{\sigma\sim\{\pm 1\}^M}\left[\sup_{\tilde{f}\in\tilde{\gF}}\langle \sigma,(g_{\tilde{f}})_{|\gS}\rangle\right]\le \sqrt{2}L \E_{\sigma\sim\{\pm 1\}^{nM}}\left[\sup_{\tilde{f}\in\tilde{\gF}}\langle \sigma,\tilde{f}_{|\gS}\rangle\right],
\end{equation}
where $\tilde{f}_{|\gS}=\left(\tilde{f}_t(z_j)\right)_{t\in[n],j\in[M]}$ and all signs are independent Rademacher variables.
\end{lemma}
\begin{proof}
    Apply Corollary 4 of \citet{maurer2016vector} to the finite-sample restrictions $\{\tilde{f}_{|\gS}:\tilde{f}\in\tilde{\gF}\}$ with the Euclidean Lipschitz constant $L$. The independent coordinate signs in that corollary give exactly the right-hand side. If $h$ is only defined on the relevant representation domain, a scalar $L$-Lipschitz extension to $\sR^n$ gives the same values on those restrictions.
\end{proof}

\begin{theorem}
\label{theorem:generalization-bound-rademacher}
Let the loss on representation tuples be $L$-Lipschitz with respect to the Euclidean norm and take values in $[a,a+B]$ for some $B>0$ on the domain where each representation has norm at most $R$. For $\delta\in(0,1)$, with probability at least $1-\delta$ over $M$ i.i.d. training tuples $\gS=\{(x_j,x_j^+,x_{j1}^-,\ldots,x_{jk}^-)\}_{j=1}^M$, the empirical minimizer $\widehat{f}$ satisfies, for all $f\in\gF$,
\begin{equation}
    \Lcon(\widehat{f})\le\Lcon(f)+O\left(\frac{L\gR_\gS(\gF)}{M}+B\sqrt{\frac{\log(2/\delta)}{M}}\right),
\end{equation}
where $\gR_\gS(\gF)=\E_{\sigma\sim\{\pm1\}^{(k+2)dM}}[\sup_{f\in\gF}\langle\sigma,f_{|\gS}\rangle]$ and $f_{|\gS}=\left(f_t(x_j),f_t(x_j^+),f_t(x_{j1}^-),\ldots,f_t(x_{jk}^-)\right)_{j\in[M],t\in[d]}$.
\end{theorem}
\begin{proof}
    First, we use the classical empirical Rademacher bound. For a class $\gG$ of functions from $Z$ to $[0,1]$ and an i.i.d. sample $\gS=\{z_j\}_{j=1}^M$, with probability at least $1-\delta/2$, for all $g\in\gG$,
    \begin{equation*}
        \E[g(z)]\le\frac1M\sum_{j=1}^M g(z_j)
        +\frac{2\gR_\gS(\gG)}{M}+3\sqrt{\frac{\log(4/\delta)}{2M}}.
    \end{equation*}
    Here $\gR_\gS(\gG)$ is unnormalized. Apply this bound with $Z=\gX^{k+2}$ and
    \begin{equation*}
        \gG=\left\{g_f(z)=\frac{\ell(f,x,x^+,\{x_i^-\}_{i=1}^k)-a}{B}:f\in\gF\right\}.
    \end{equation*}
    Apply it also to $1-\gG$ and use a union bound. These two classes have the same empirical Rademacher complexity by symmetry of the signs.

    To bound this complexity, set $n=d(k+2)$ and
    \begin{equation*}
        \tilde{\gF}=\left\{\tilde{f}(x,x^+,x_1^-,\ldots,x_k^-)
        =(f(x),f(x^+),f(x_1^-),\ldots,f(x_k^-)):f\in\gF\right\}.
    \end{equation*}
    The normalized loss is $L/B$-Lipschitz on the representation domain. Lemma \ref{lipschitz-rademacher} gives $\gR_\gS(\gG)\le\sqrt{2}(L/B)\gR_\gS(\gF)$. Therefore, with probability at least $1-\delta$, simultaneously for all $f\in\gF$,
    \begin{equation}
        \label{eq:generalization-bound-1}
        \left|\Lcon(f)-\widehat{L}_{\mathrm{con}}(f)\right|
        \le\Gamma_M\triangleq\frac{2\sqrt{2}L\gR_\gS(\gF)}{M}
        +3B\sqrt{\frac{\log(4/\delta)}{2M}}.
    \end{equation}
    Finally, empirical optimality gives $\Lcon(\widehat{f})\le\widehat{L}_{\mathrm{con}}(\widehat{f})+\Gamma_M\le\widehat{L}_{\mathrm{con}}(f)+\Gamma_M\le\Lcon(f)+2\Gamma_M$ for every $f\in\gF$, completing the proof. Substituting $L\le\eta\zeta_1(R)$ and $B\le\zeta_2(R)$ gives $2\Gamma_M=\Gen_M$ up to universal constants.
\end{proof}

The Rademacher complexity $\gR_\gS(\gF)$ depends on $k$ through the representation tuples.

For a concrete capacity bound, consider $f_A(x)=Ax$ with $\|A\|_F\le\Lambda$, $\|x\|\le R_x$, and $\Lambda R_x\le R$. Enumerate the $M(k+2)$ input occurrences in $\gS$ as $z_1,\ldots,z_{M(k+2)}$, and let $\sigma_j\in\{\pm1\}^d$ have independent coordinates. Frobenius duality and Cauchy--Schwarz give
\begin{equation}
    \label{eq:linear-rademacher-bound}
    \gR_\gS(\gF)
    =\Lambda\E_\sigma\left\|\sum_{j=1}^{M(k+2)}\sigma_jz_j^\top\right\|_F
    \le\Lambda\sqrt{d\sum_{j=1}^{M(k+2)}\|z_j\|^2}
    \le\Lambda R_x\sqrt{dM(k+2)}.
\end{equation}

% \begin{lemma}
% With probability at least $1-\delta$ over the training set $\gS$, for all $f\in\gF$, we have
% \begin{equation}
%     \Lcon(\widehat{f})\le \Lcon(f)+Gen_{M}.
% \end{equation}
% \end{lemma}

\subsubsection{Proof of Corollary \ref{cor:classification-error}}

\begin{proof}
    Fix $(x,y)$ and write $g_c=w_c^\top f(x)$. Grouping the $k$ negative classes by their values gives
    \begin{equation*}
        h(Wf(x),y)=\E\log\left[1+k_y+\sum_{c\neq y}k_ce^{g_c-g_y}\right]\ge h_0(y).
    \end{equation*}
    If $\hat{y}(x)\neq y$, then $g_c\ge g_y$ for $c=\hat{y}(x)$, and dropping the other nonnegative terms gives $h(Wf(x),y)\ge\E\log(1+k_y+k_c)=h_0(y)+\E\log(1+\frac{k_c}{1+k_y})\ge h_0(y)+\omega_k(y)$. Hence $h(Wf(x),y)-h_0(y)\ge\omega_k(y)\mathbb{I}(\hat{y}(x)\neq y)$. Taking expectations over $x\sim\gD^y$ and summing over $y$ with weights $\rT(y)$ yield
    \begin{equation*}
        \Lsup(\gD_{\gT},Wf,h)-\E_{y\sim\rT}h_0(y)\ge\sum_{y\in\gT}\rT(y)\omega_k(y)\sP_{x\sim\gD^y}\left[\hat{y}(x)\neq y\right].
    \end{equation*}
    Lower bounding $\omega_k(y)$ by $\min_{y'\in\gT}\omega_k(y')$, or $\rT(y)\omega_k(y)$ by $\min_{y'\in\gT}\rT(y')\omega_k(y')$, proves the two bounds. The final claim follows from Theorem \ref{thm:generalization-bound-for-contrastive-learning} with $\eta=R$.
\end{proof}

\subsubsection{Proof of Theorem \ref{thm:bound-supervised-CL}}
\label{sec:proof-supervised-CL}

\begin{proof}
    Let $w_c=\E_{x\sim\gD^c}\widehat{f}(x)$, and let $W\in\sR^{|\gT|\times d}$ have rows $w_c^\top$. Class-conditional independence of positive pairs gives $\Disc(\widehat{f})=0$. Conditioning on the anchor and the sampled negative classes, Jensen's inequality gives
    \begin{equation*}
        \begin{aligned}
            \Lconsup(\widehat{f},\gD;\ell)&\ge \E_{(x,y)\sim\gD_{\gT}}\E_{\{c_i\}_{i=1}^k\iidsim\rho_{\gT}^{\not y}}\ell\left(\widehat{f}(x),w_y,\{w_{c_i}\}_{i=1}^k\right)\\
            &=\Lsup(\gD_{\gT},W\widehat{f},h),
        \end{aligned}
    \end{equation*}
    where $h$ is the incurred supervised loss stated in the theorem.

    For $M$ i.i.d. supervised training tuples $\gS$, the argument based on Rademacher complexity and vector contraction in Theorem \ref{theorem:generalization-bound-rademacher} applies to the empirical minimizer $\widehat{f}$. Thus, with probability at least $1-\delta$, for all $f\in\gF$,
    \begin{equation*}
        \Lconsup(\widehat{f},\gD;\ell)\le \Lconsup(f,\gD;\ell)+\Gen_M,
    \end{equation*}
    where $\Gen_M$ is as in Theorem \ref{thm:generalization-bound-for-contrastive-learning}, with the empirical Rademacher complexity evaluated on $\gS$. Combining the inequalities completes the proof.
\end{proof}

\subsubsection{Proof of Proposition \ref{prop:regularized-bound}}
\label{sec:proof-regularized-bound}

\begin{proof}
    We first bound the estimation error of $\widehat{\Disc}_i$ uniformly over $\gF$. For $f\in\gF$, write $\mu_c=\E_{x\sim\gD^c}f(x)$ and $I_c=\{n:y_{x_n}=c\}$, and let $\widetilde{\Disc}_i(f)=\frac{1}{N}\sum_{n=1}^N\|f(x_n)-\mu_{y_{x_n}}\|$, whose expectation is $\Disc_i(f)$. Since $N_c(\widehat{\mu}_c-\mu_c)=\sum_{n\in I_c}(f(x_n)-\mu_c)$, the triangle inequality and $\|v\|=\sup_{\|u\|\le1}u^\top v$ give
    \begin{equation*}
        \left|\widehat{\Disc}_i(f)-\widetilde{\Disc}_i(f)\right|
        \le\frac{1}{N}\sum_{c\in\gT}\Bigl\|\sum_{n\in I_c}(f(x_n)-\mu_c)\Bigr\|
        =\sup_{\{u_c\}_{c\in\gT}\subseteq\gB}\frac{1}{N}\sum_{n=1}^Nu_{y_{x_n}}^\top\left(f(x_n)-\mu_{y_{x_n}}\right),
    \end{equation*}
    where $\gB$ is the unit ball in $\sR^d$ and every summand has zero mean. We apply the classical empirical Rademacher bound in the proof of Theorem \ref{theorem:generalization-bound-rademacher}, with $\delta/3$ in place of $\delta$, to $\gG_1=\{\|f(x)-\mu_{y_x}\|:f\in\gF\}$ in both directions and to $\gG_2=\{u_{y_x}^\top(f(x)-\mu_{y_x}):f\in\gF,\{u_c\}_{c\in\gT}\subseteq\gB\}$ in one direction. Their values lie in intervals of length $2R$ and $4R$, so with probability at least $1-\delta/2$, for all $f\in\gF$,
    \begin{equation*}
        \left|\widehat{\Disc}_i(f)-\Disc_i(f)\right|
        \le\frac{2\gR_{\gS'}(\gG_1)+2\gR_{\gS'}(\gG_2)}{N}+18R\sqrt{\frac{\log(12/\delta)}{2N}}.
    \end{equation*}
    To bound these complexities, let $\sigma_n,\sigma_n'\in\{\pm1\}^d$ have independent Rademacher coordinates. Since $\|\mu_c\|\le R$, $\E_\sigma\|\sum_{n\in I_c}\sigma_n\|\le\sqrt{dN_c}$, and $\sum_{c\in\gT}\sqrt{N_c}\le\sqrt{|\gT|N}$, we have
    \begin{equation*}
        \E_\sigma\sup_{f\in\gF}\sum_{n=1}^N\sigma_n^\top\left(f(x_n)-\mu_{y_{x_n}}\right)
        \le\gR_{\gS'}(\gF)+R\sum_{c\in\gT}\E_\sigma\Bigl\|\sum_{n\in I_c}\sigma_n\Bigr\|
        \le\gR_{\gS'}(\gF)+R\sqrt{d|\gT|N},
    \end{equation*}
    and similarly $\E_{\sigma'}\sup_{\{u_c\}_{c\in\gT}\subseteq\gB}\sum_{n=1}^N2R\sigma_n'^\top u_{y_{x_n}}\le2R\sqrt{d|\gT|N}$. The classes $\gG_1$ and $\gG_2$ are obtained by applying the $1$-Lipschitz map $v\mapsto\|v\|$ to $f(x)-\mu_{y_x}$ and the map $(a,b)\mapsto a^\top b/(2R)$, which is $\sqrt{2}$-Lipschitz on $\{\|a\|,\|b\|\le2R\}$, to $(2Ru_{y_x},f(x)-\mu_{y_x})$. By Lemma \ref{lipschitz-rademacher}, $\gR_{\gS'}(\gG_1)\le\sqrt{2}(\gR_{\gS'}(\gF)+R\sqrt{d|\gT|N})$ and $\gR_{\gS'}(\gG_2)\le2(\gR_{\gS'}(\gF)+3R\sqrt{d|\gT|N})$. Hence, with probability at least $1-\delta/2$,
    \begin{equation*}
        \sup_{f\in\gF}\left|\widehat{\Disc}_i(f)-\Disc_i(f)\right|
        \le\Gamma_N\triangleq\frac{7\gR_{\gS'}(\gF)+15R\sqrt{d|\gT|N}}{N}+18R\sqrt{\frac{\log(12/\delta)}{2N}}.
    \end{equation*}

    By Equation \ref{eq:generalization-bound-1} with $\delta/2$ in place of $\delta$, with probability at least $1-\delta/2$, $\sup_{f\in\gF}|\Lcon(f)-\widehat{L}_{\mathrm{con}}(f)|\le\Gamma_M$. On the intersection of both events, the definition of $\widehat{f}$ implies, for every $f\in\gF$,
    \begin{equation*}
        \begin{aligned}
            \Lcon(\widehat{f})+\eta\Disc_i(\widehat{f})
            &\le\widehat{L}_{\mathrm{con}}(\widehat{f})+\eta\widehat{\Disc}_i(\widehat{f})+\Gamma_M+\eta\Gamma_N\\
            &\le\widehat{L}_{\mathrm{con}}(f)+\eta\widehat{\Disc}_i(f)+\Gamma_M+\eta\Gamma_N\\
            &\le\Lcon(f)+\eta\Disc_i(f)+2\Gamma_M+2\eta\Gamma_N.
        \end{aligned}
    \end{equation*}
    By the triangle inequality and Jensen's inequality, $\Disc(\widehat{f})\le\Disc_a(\widehat{f})+\Disc_i(\widehat{f})$ as in Section \ref{sec:internal-discrepancy}. Applying Theorem \ref{thm:bound-for-contrastive-losses} to $\widehat{f}$, we obtain
    \begin{equation*}
        \Lsup(\gD_{\gT},W\widehat{f},h)\le\Lcon(\widehat{f})+\eta\Disc(\widehat{f})
        \le\Lcon(\widehat{f})+\eta\Disc_i(\widehat{f})+\eta\Disc_a(\widehat{f}).
    \end{equation*}
    Combining it with the previous display, taking the infimum over $f\in\gF$, and noting that $2\Gamma_M=\Gen_M$ and $2\Gamma_N=\Gen_N$ up to universal constants complete the proof.
\end{proof}

\subsubsection{Proof of Proposition \ref{prop:relaxed-kmeans}}

\begin{proof}
    Let $I_c=\{n:y_{x_n}=c\}$ for $c\in\gT$, and split $\widehat{L}_{\mathrm{k\text{-}means}}(f)=\widehat{L}_{\mathrm{w}}(f)+\widehat{L}_{\mathrm{b}}(f)$, where $\widehat{L}_{\mathrm{w}}(f)$ sums over the pairs $(i,j)$ from the same class and $\widehat{L}_{\mathrm{b}}(f)$ over the $N^2-\sum_{c\in\gT}N_c^2$ pairs from different classes. Each summand of $\widehat{L}_{\mathrm{b}}(f)$ lies in $[0,\tau]$, so $0\le\widehat{L}_{\mathrm{b}}(f)\le\tau(N^2-\sum_{c\in\gT}N_c^2)$, and the upper bound holds with equality under the additional condition. It therefore suffices to show that
    \begin{equation*}
        \frac{\tau N_{\min}N}{2R}\widehat{\Disc}_i(f)\le\widehat{L}_{\mathrm{w}}(f)\le2N_{\max}N\widehat{\Disc}_i(f).
    \end{equation*}
    Since $\|f(x)\|\le R$, all pairwise distances lie in $[0,2R]$, on which $\min\{a,\tau\}\ge\tau a/(2R)$ as $\tau\le2R$. For $i\in I_c$, Jensen's inequality gives $\sum_{j\in I_c}\|f(x_i)-f(x_j)\|\ge N_c\|f(x_i)-\widehat{\mu}_c\|$. Hence
    \begin{equation*}
        \widehat{L}_{\mathrm{w}}(f)\ge\frac{\tau}{2R}\sum_{c\in\gT}N_c\sum_{i\in I_c}\|f(x_i)-\widehat{\mu}_c\|\ge\frac{\tau N_{\min}N}{2R}\widehat{\Disc}_i(f).
    \end{equation*}
    Conversely, for $i,j\in I_c$, the triangle inequality gives $\min\{\|f(x_i)-f(x_j)\|,\tau\}\le\|f(x_i)-\widehat{\mu}_c\|+\|f(x_j)-\widehat{\mu}_c\|$. Summing over $i,j\in I_c$ and $c\in\gT$ yields
    \begin{equation*}
        \widehat{L}_{\mathrm{w}}(f)\le2\sum_{c\in\gT}N_c\sum_{i\in I_c}\|f(x_i)-\widehat{\mu}_c\|\le2N_{\max}N\widehat{\Disc}_i(f),
    \end{equation*}
    which completes the proof.
\end{proof}

\subsection{Implications for Mitigating Class Imbalance}

\subsubsection{Proof of Theorem \ref{thm:approximation-infonce-balanced-softmax}}
\label{sec:balanced-risk-interpretation}
To prove that $h=\tilde{h}+\log k + O(k^{-1})$ under the condition that $\|f(x)\|,\|w_c\|\le R$, we introduce $\overline{h}(x,y)\triangleq\log[1+k\E_{c\sim\rT}e^{f(x)^\top(w_c-w_y)}]$. We bound the Jensen gap using a second-order bound on the logarithm, as detailed in Lemma \ref{lemma:approximation-error} below. Its lower bound gives $\tilde{h}+\log k\le h+V_R/k$, where $V_R=(e^{R^2}-1)^2/2$. Combining this with Theorem \ref{thm:generalization-bound-for-contrastive-learning}, we can deduce
\begin{equation*}
    \Lsup(\gD_{\gT},W\widehat{f},\tilde{h})+\log k\le \Lcon(f)+\eta\cdot\Disc(\widehat{f})+\Gen_{M}+\frac{V_R}{k}.
\end{equation*}

\begin{lemma}
\label{lemma:approximation-error}
Assume $\|f(x)\|,\|w_c\|\le R$. For $\overline{h}(x,y)=\log[1+k\E_{c\sim\rT}e^{f(x)^\top(w_c-w_y)}]$, we denote its (biased) Monte Carlo estimate with $k$ i.i.d. random samples $c_i\sim\rT$, $i=1,\ldots,k$, by $\hat{h}(x,y,\{c_i\}_{i=1}^k)\triangleq \log\left[1+\sum_{i=1}^k e^{f(x)^\top(w_{c_i}-w_y)}\right]$. Then $h(x,y)=\E_{\{c_i\}_{i=1}^k\iidsim\rho_{\gT}}\hat{h}(x,y,\{c_i\}_{i=1}^k)$. Let $U_R=e^{2R^2}$ and $V_R=(e^{R^2}-1)^2/2$. Then
\begin{equation}
    -\frac{V_R}{k}\le h(x,y)-\log k-\tilde{h}(x,y)\le\frac{\min\{U_R,1/\rT(y)\}}{k}.
\end{equation}
Consequently, since $V_R<U_R$, $|h-\log k-\tilde{h}|\le U_R/k$.
\end{lemma}
\begin{proof}
    Fix $(x,y)$ and let $Z_i=e^{f(x)^\top(w_{c_i}-w_y)}$, $m=\E Z_1$, and $\bar{Z}=k^{-1}\sum_{i=1}^k Z_i$. Writing $s=f(x)^\top w_y$, $a=e^{-R^2-s}$, and $b=e^{R^2-s}$, the norm constraints imply $|f(x)^\top w_c|\le R^2$, so $a\le Z_i\le b$ with $b/a=e^{2R^2}$. Hence
    \begin{equation*}
        h-\log k=\E\log(\bar{Z}+1/k),\qquad \tilde{h}=\log m.
    \end{equation*}
    Let $X=\bar{Z}+1/k$ and $\nu=m+1/k$, so that $\E X=\nu$ and $X,\nu\ge a+1/k$. For $u>0$, integrating $1/s-1$ between $u$ and $1$ gives $-\log u\le1-u+\frac{(1-u)^2}{2\min\{u,1\}}$. Applying it with $u=X/\nu$ and taking expectations, the linear term vanishes. Together with Jensen's inequality, this gives
    \begin{equation}
    \label{eq:approximation-error-1}
    \begin{aligned}
        0\le J&\triangleq\overline{h}-h=\log(m+1/k)-\E\log(\bar{Z}+1/k)
        \le\E\frac{(X-\nu)^2}{2\nu\min\{X,\nu\}}\\
        &\le\frac{\operatorname{Var}(Z_1)}{2kma}
        \le\frac{(b-m)(m-a)}{2kma}
        \le\frac{(\sqrt{b}-\sqrt{a})^2}{2ka}=\frac{V_R}{k},
    \end{aligned}
    \end{equation}
    where independence gives $\operatorname{Var}(X)=\operatorname{Var}(Z_1)/k$, taking expectations in $(b-Z_1)(Z_1-a)\ge0$ gives $\operatorname{Var}(Z_1)\le(b-m)(m-a)$, and $m+ab/m\ge2\sqrt{ab}$ gives the last inequality. Also, using $m\ge a\ge e^{-2R^2}$ and $m\ge\rT(y)$, which holds as $Z_i=1$ whenever $c_i=y$, we have
    \begin{equation}
    \label{eq:approximation-error-2}
        0\le\log(m+1/k)-\log m
        =\log\left(1+\frac{1}{km}\right)
        \le\frac{1}{km}\le\frac{\min\{U_R,1/\rT(y)\}}{k}.
    \end{equation}
    Since $h-\log k-\tilde{h}=\log(1+1/(km))-J$, Equations \ref{eq:approximation-error-1} and \ref{eq:approximation-error-2} prove the claimed bounds.
\end{proof}

The order $k^{-1}$ cannot in general be improved under these assumptions: when all logits are zero, $Z_i\equiv1$ and $h-\log k-\tilde{h}=\log(1+1/k)\sim k^{-1}$.

% For the interpretation under class imbalance, fix a representation $z=f(x)$ and write $p_c=\rT(c)$ and $q_c(z)=\sP[y=c\mid f(x)=z]$ under the training distribution. For unrestricted scores $s_c(z)$, the conditional balanced softmax risk, up to a score-independent constant, is $\log\sum_c p_ce^{s_c}-\sum_cq_cs_c$. Differentiation shows that its minimum, when all $q_c>0$, is attained at
% \begin{equation*}
%     s_c(z)=\log\frac{q_c(z)}{p_c}+b(z),
% \end{equation*}
% where $b(z)$ is common to all classes; zero posterior probabilities are handled by limits. If the class-conditional distributions are fixed and the test prior is uniform, its posterior is proportional to $q_c(z)/p_c$, so $\argmax_c s_c(z)$ minimizes the balanced classification error $|\gT|^{-1}\sum_c\sP_{x\sim\gD^c}[\argmax_j s_j(f(x))\ne c]$. Under the original prior, the corresponding decision is $\argmax_c\{s_c(z)+\log p_c\}$. These statements concern unrestricted scores of fixed representations and need not be attained by the norm-bounded mean classifier.

\subsubsection{Proof of Theorem \ref{thm:balanced-SpeCL}}
\begin{proof}
The SpeCL-incurred supervised loss is formulated as $h(x,y)=-2f(x)^\top w_y+\E_{c\sim\rT}(f(x)^\top w_c)^2$.

Let $\mu_c=\E_{x\sim\gD^c}f(x)$ for $c\in\gT$. Then the population risk is 
\begin{equation*}
\begin{aligned}
\gL=&\E_{(x,y)\sim\gD_{\gT}}h(x,y)=\E_{(x,y)\sim\gD_{\gT}}\left[-2f(x)^\top w_y+\E_{c\sim\rT}(f(x)^\top w_c)^2\right]\\
\ge &\E_{c\sim\rT}\left[-2\mu_c^\top w_c + \E_{c'\sim\rT}(\mu_c^\top w_{c'})^2\right]\\
= & -2\sum_{c\in\gT}\rT(c)\mu_c^\top w_c+\sum_{c,c'\in\gT}\rT(c)\rT(c')(\mu_c^\top w_{c'})^2\\
=&\sum_{c\in\gT}\left[\rT^2(c)(\mu_c^\top w_c)^2-2\rT(c)\mu_c^\top w_c\right]+\sum_{c\neq c',c,c'\in\gT}\rT(c)\rT(c')(\mu_c^\top w_{c'})^2\\
=& \sum_{c\in\gT}\left[\rT(c)\mu_c^\top w_c-1\right]^2+\sum_{c\neq c',c,c'\in\gT}\rT(c)\rT(c')(\mu_c^\top w_{c'})^2-|\gT|\\
=& \left\|\left(\sqrt{\rT(c)}\mu_c\right)_{c\in\gT}^\top\left(\sqrt{\rT(c)}w_c\right)_{c\in\gT}-\mI_{|\gT|}\right\|_F^2-|\gT|.
\end{aligned}
\end{equation*}
Let $C=|\gT|$ and $r=\min\{d,C\}$. Since the matrix product above has rank at most $d$, $\gL\ge-r$. This lower bound is attained: choose $r$ classes $c_1,\ldots,c_r$, set $f(x)=e_j$ for $x\sim\gD^{c_j}$ and $w_{c_j}=e_j/\rT(c_j)$, and set the representations and classifier weights of the remaining classes to zero. Here $e_1,\ldots,e_r$ are standard basis vectors. Therefore, the minimum risk is $-r$ for every strictly positive class distribution, including the uniform distribution.

Writing $\gL_{\rT}(f,W)$ for the risk under $\rT$ and $\gL_{\mathrm{unif}}(f,W)$ for the risk under the uniform distribution, direct substitution of the stated transformations gives, for every pair $(\tilde{f},\tilde{W})$,
\begin{equation*}
    \begin{aligned}
        \gL_{\rT}(\hat{f},\hat{W})
        =&-\frac{2}{C}\sum_{c\in\gT}\E_{x\sim\gD^c}\tilde{f}(x)^\top\tilde{w}_c\\
        &+\frac{1}{C^2}\sum_{c,c'\in\gT}\E_{x\sim\gD^c}\bigl(\tilde{f}(x)^\top\tilde{w}_{c'}\bigr)^2\\
        =&\gL_{\mathrm{unif}}(\tilde{f},\tilde{W}).
    \end{aligned}
\end{equation*}
For a global minimizer under the uniform distribution, this value is $-r$, so the transformed pair is a global minimizer for every $s>0$. The stated identity for the logits follows by multiplying the two scaling factors for $x\sim\gD^c$.

Conversely, let $(\hat{f},\hat{W})$ be a global minimizer under $\rT$ and let $s>0$. Define $\tilde{f}(x)=\frac{\sqrt{\rT(c)C}}{s}\hat{f}(x)$ for $x\sim\gD^c$ and $\tilde{w}_c=s\sqrt{\rT(c)C}\,\hat{w}_c$ for $c\in\gT$. Then the identity above gives $\gL_{\mathrm{unif}}(\tilde{f},\tilde{W})=\gL_{\rT}(\hat{f},\hat{W})=-r$. Hence $(\tilde{f},\tilde{W})$ is a global minimizer under the uniform distribution, and its transformation with this $s$ is exactly $(\hat{f},\hat{W})$.
\end{proof}

\subsubsection{Proof of Proposition \ref{prop:normalized-SpeCL}}
\begin{proof}
    Fix $(x,y)$ and write $u_c=f(x)^\top w_c$ for $c\in\gT$. By the Cauchy--Schwarz inequality, $u_y\le\|f(x)\|\|w_y\|\le R^2\le1/\rT(y)$. Since $u\mapsto-2u+\rT(y)u^2$ is strictly decreasing on $(-\infty,1/\rT(y)]$, we have
    \begin{equation*}
        h(x,y)=-2u_y+\sum_{c\in\gT}\rT(c)u_c^2
        \ge-2u_y+\rT(y)u_y^2
        \ge-2R^2+\rT(y)R^4.
    \end{equation*}
    Taking expectations gives $\E_{(x,y)\sim\gD_{\gT}}h(x,y)\ge-2R^2+R^4\sum_{c\in\gT}\rT^2(c)$. Equality holds if and only if, for almost every $x$, $u_c=0$ for all $c\neq y_x$ and $u_{y_x}=R^2$. By the equality condition of the Cauchy--Schwarz inequality, $u_{y_x}=R^2$ holds if and only if $\|w_{y_x}\|=R$ and $f(x)=w_{y_x}$, in which case $u_c=w_{y_x}^\top w_c$. Since $\rT(c)>0$ for all $c\in\gT$, equality therefore holds if and only if $\{w_c/R\}_{c\in\gT}$ is orthonormal and $f(x)=w_{y_x}$ almost surely. Such minimizers exist since $d\ge|\gT|$. As they satisfy $\|f(x)\|=\|w_c\|=R$, the same characterization and minimum hold under the normalization.
\end{proof}

\subsection{The Optimality of Supervised Losses Incurred by Contrastive Learning}

\subsubsection{Proof of Theorem \ref{thm:optimality-of-InfoNCE}}

\begin{proof}
    With the mean classifier, $\Lsup$ is the risk $\gH(f)$ in Lemma \ref{lemma:infonce-ce}. The proof of that lemma shows that $f$ minimizes $\gH(f)$ if and only if $f(x)=\mu_{y_x}$ almost surely, where $\|\mu_c\|=R$ and $\mu_c^\top\mu_{c'}=-R^2/(|\gT|-1)$ for $c\neq c'$, which is the claim for $\Lsup$. The case of $\Lsup'$ is proved in Appendix \ref{sec:optimality-supervised-cases}.
\end{proof}

\subsubsection{Proof of Corollary \ref{cor:near-optimal-infonce}}
\label{sec:proof-near-optimal-infonce}

\noindent\textbf{Proof}\ Write $C=|\gT|$, $a=C/(C-1)$, and define
    \begin{equation*}
        \begin{aligned}
            F_{\mathrm{U}}(z)&=\sum_{j=0}^k\binom{k}{j}\frac{(C-1)^j}{C^k}
                \log[k-j+1+je^{az}],\\
            F_{\mathrm{S}}(z)&=\log[1+ke^{az}].
        \end{aligned}
    \end{equation*}
    Use $F=F_{\mathrm{U}}$ for $\Lsup$ and $F=F_{\mathrm{S}}$ for $\Lsup'$. Set
    \begin{equation*}
        D=R^2-\frac1C\sum_c\|\mu_c\|^2+\|\bar{\mu}\|^2,
        \qquad \kappa=F'(-R^2)>0.
    \end{equation*}
    The Jensen inequalities in the proof of Lemma \ref{lemma:infonce-ce}, with all $k$ negatives from other classes for $\Lsup'$, bound the respective risk below by $F(-R^2+D)$. The feasible simplex ETF attains $F(-R^2)$. Convexity of $F$ therefore gives $\eps\ge\kappa D$. Since
    \begin{equation*}
        \frac1C\sum_c\E_{x\sim\gD^c}\|f(x)-\mu_c\|^2+\|\bar{\mu}\|^2
        \le D,
    \end{equation*}
    the first bound follows from $\kappa\ge e^{-2R^2}/2$. Indeed, each denominator in $F'(-R^2)$ is at most $k+1$, and $j$ has mean $k(C-1)/C$ under the binomial weights of $F_{\mathrm{U}}$, so $\kappa\ge ke^{-aR^2}/(k+1)$ for both choices of $F$.

    For the Gram matrix, let $G=(\mu_c^\top\mu_{c'})_{c,c'}$ and $G_{\mathrm{ETF}}=aR^2(\mI_C-\vone\vone^\top/C)$. Write $\Psi(G)$ for the corresponding risk with the anchor representation replaced by its class mean. Jensen's inequality bounds the original risk below by $\Psi(G)$, with equality at $G_{\mathrm{ETF}}$. Every matrix on the segment joining $G_{\mathrm{ETF}}$ and $G$ is positive semidefinite with diagonal at most $R^2$. For any such matrix, fix a class vector of norm $r\le R$ in a Gram factorization and write $s_i\in[-rR,rR]$ for its sampled inner products, with margins $z_i=s_i-r^2$. Since $Z=e^{r^2}+\sum_i e^{s_i}\le(k+1)e^{rR}$, the softmax weights $q_0=e^{r^2}/Z$ and $q_i=e^{s_i}/Z$ satisfy
    \begin{equation*}
        q_0q_i\ge\frac{e^{r^2-3rR}}{(k+1)^2}\ge\frac{e^{-2R^2}}{(k+1)^2},
        \qquad q_iq_j\ge\frac{e^{-4R^2}}{(k+1)^2}.
    \end{equation*}
    For $\psi(z)=\log(1+\sum_i e^{z_i})$ and the margin direction $\xi_i=\Delta_{cc_i}$ with $\Delta$ defined below, the identity
    \begin{equation*}
        \xi^\top\nabla^2\psi(z)\xi
        =q_0\sum_iq_i\xi_i^2+\sum_{i<j}q_iq_j(\xi_i-\xi_j)^2
    \end{equation*}
    gives, after i.i.d. negative-class sampling, an expected curvature of at least $\gamma_k\operatorname{Var}(\xi_1)$, where
    \begin{equation*}
        \gamma_k=\frac{ke^{-2R^2}+k(k-1)e^{-4R^2}}{(k+1)^2}.
    \end{equation*}
    Here $\E(\xi_i-\xi_j)^2=2\operatorname{Var}(\xi_1)$ and $\E\xi_1^2\ge\operatorname{Var}(\xi_1)$; the $q_0$ term also covers $k=1$.

    Set $\Delta_{cc'}=G_{cc'}-G_{cc}+aR^2$ for $c\ne c'$, and $\Delta_{cc}=0$. Let $\bar{\Delta}_c=(C-1)^{-1}\sum_{c'\ne c}\Delta_{cc'}$ and $V=\sum_{c\ne c'}(\Delta_{cc'}-\bar{\Delta}_c)^2$. Since $\sum_{c\ne c'}\Delta_{cc'}=C^2D$, symmetry gives the linear term $\kappa D$ in the expansion of $\Psi$ at $G_{\mathrm{ETF}}$. For uniform sampling over all classes, the variance of $\Delta_{c c_i}$ is at least $(C-1)/C$ times its variance over the other classes. Taylor's formula along the Gram segment consequently gives, for both sampling schemes,
    \begin{equation*}
        \eps\ge\Psi(G)-\Psi(G_{\mathrm{ETF}})
        \ge\kappa D+\frac{\gamma_k}{2C^2}V.
    \end{equation*}
    Finally, let $\delta_c=R^2-G_{cc}\ge0$ and $u_c=\mu_c^\top\bar{\mu}$. The off-diagonal row mean of $G-G_{\mathrm{ETF}}$ is $(Cu_c+\delta_c)/(C-1)$, while $\sum_c(\delta_c^2+u_c^2)\le CR^2D$. Decomposing each row into its mean and centered part and applying Cauchy--Schwarz yields
    \begin{equation*}
        \begin{aligned}
            \|G-G_{\mathrm{ETF}}\|_F^2
            &=V+\sum_c\delta_c^2+\frac1{C-1}\sum_c(Cu_c+\delta_c)^2\\
            &\le V+\frac{C^2(C+1)}{C-1}R^2D.
        \end{aligned}
    \end{equation*}
    Combining these two bounds gives
    \begin{equation}
    \label{eq:explicit-gram-stability}
        \frac{\|G-G_{\mathrm{ETF}}\|_F^2}{C^2}
        \le\max\left\{\frac{2}{\gamma_k},\frac{C+1}{C-1}\frac{R^2}{\kappa}\right\}\eps.
    \end{equation}
    Since $\gamma_k\ge e^{-4R^2}/4$ and $\kappa\ge e^{-aR^2}/2$ for both choices of $F$, the coefficient is at most $8e^{4R^2}$ for all $k\ge1$.

    For Equation \ref{eq:learned-infonce-excess}, choose a class-constant simplex ETF $f_0\in\gF$ as the comparator in Theorem \ref{thm:generalization-bound-for-contrastive-learning}. At $f_0$, replacing negatives by their class means is exact. Replacing the positive representation by its target-class mean changes the loss only when $y_{x^+}\ne y_x$, and then by at most $\sqrt{2a}R^2$, since InfoNCE is $R$-Lipschitz in that representation and distinct class means of $f_0$ are at distance $\sqrt{2a}R$. Since $x^+$ follows the marginal distribution $\gD$, $f_0$ is class-constant almost surely for both views. Thus $\Lcon(f_0)\le\Lsup^*+\sqrt{2a}R^2\alpha$, and substitution into Theorem \ref{thm:generalization-bound-for-contrastive-learning} gives Equation \ref{eq:learned-infonce-excess}. Both stability bounds above apply with $\eps=R\Disc(\widehat{f})+\Gen_M+\sqrt{2a}R^2\alpha$.
\noindent\hfill\BlackBox\par\vspace{2mm}

\subsubsection{Proof of Theorem \ref{thm:optimality-of-specl}}

\begin{proof}
Let $H=[\sqrt{\rT(c)}\mu_c]_{c\in\gT}\in\sR^{d\times|\gT|}$.
For the spectral contrastive loss, we have
\begin{equation*}
\begin{aligned}
\Lsup=& \E_{x\sim\gD}\E_{\{c_i\}_{i=1}^k\iidsim\rT}\left[-2f(x)^\top \mu_{y_x}+\frac{1}{k}\sum_{i=1}^k (f(x)^\top\mu_{c_i})^2\right]\\
=& \E_{x\sim\gD}\left[-2f(x)^\top \mu_{y_x}+\E_{c\sim\rT} (f(x)^\top\mu_{c})^2\right]\\
\ge&\E_{c\sim\rT}\left[-2\mu_c^\top\mu_c+\E_{c'\sim\rT}(\mu_c^\top\mu_{c'})^2\right]\\
=& -2\sum_{c\in\gT}\rT(c)\|\mu_c\|_2^2+\sum_{c\in\gT}\sum_{c'\in\gT}\rT(c)\rT(c')(\mu_c^\top\mu_{c'})^2\\
=&\left\|\mI-H^\top H\right\|_F^2-|\gT|,
\end{aligned}
\end{equation*}
where $\mI$ is the $|\gT|\times|\gT|$ identity matrix, and the inequality follows from Jensen's inequality. More precisely, its gap is
    \begin{equation*}
        \Lsup-\left(\|\mI-H^\top H\|_F^2-|\gT|\right)=\E_{x\sim\gD}\|H^\top(f(x)-\mu_{y_x})\|^2.
    \end{equation*}
    Hence the global minimum is attained by class-constant representations whose means minimize the Gram objective subject to $\|\mu_c\|\le R$. The feasible set is compact, so such means exist. Equality only requires $H^\top(f(x)-\mu_{y_x})=0$ almost surely; it does not in general imply variability collapse for every minimizer.

    Write $C=|\gT|$, $r=\min\{d,C\}$, and $G=H^\top H$. The Gram objective minus $C$ is at least $-r$, with equality if and only if $G$ is a rank-$r$ orthogonal projection. Its diagonal $q_c=G_{cc}$ must satisfy $0\le q_c\le\min\{1,R^2\rT(c)\}$ and $\sum_cq_c=r$, so the stated sum condition is necessary. Conversely, that condition allows a choice of such $q_c$. By the Schur--Horn theorem \citep{horn1954doubly}, there is a real rank-$r$ orthogonal projection with this diagonal. A Gram factorization $G=H^\top H$ in dimension $d$ gives class means with $\|\mu_c\|^2=q_c/\rT(c)\le R^2$, proving sufficiency. When $C\le d$, the rank-$C$ projection is $\mI$. Since the gap above is nonnegative, every global minimizer in this case has such a $G$ for its class means. When $C>d$, such a $G$ satisfies $\|\mI-G\|_F^2=C-d$, the smallest error of any rank-$d$ approximation of $\mI$, as all singular values of $\mI$ equal one.

    For uniform $\rT$, let $s=\operatorname{tr}(G)\le R^2$. Since $\operatorname{rank}(G)\le r$, Cauchy--Schwarz gives $\|G\|_F^2\ge s^2/r$, and the risk is at least $-2s+s^2/r$. Its minimum over $0\le s\le R^2$ is attained at $s=\min\{R^2,r\}$. Again by Schur--Horn, choose a rank-$r$ orthogonal projection $P$ with constant diagonal $r/C$. Then $G=\min\{R^2/r,1\}P$ has diagonal at most $R^2/C$ and attains this lower bound, giving the stated minimum.
\end{proof}

\subsubsection{Optimality of Supervised Contrastive Cases}
\label{sec:optimality-supervised-cases}
The supervised contrastive loss $\Lconsup(f)$ is formulated as
\begin{equation*}
\Lconsup(f)=\E_{c\sim \rT}\E_{\substack{x,x^+\iidsim \gD^c\\\{x_i^-\}_{i=1}^k\iidsim\gD^{\not c}}}\ell(f(x),f(x^+),\{f(x^-_i)\}_{i=1}^k),
\end{equation*}
and its incurred supervised loss is formulated as
\begin{equation*}
\Lsup(\gD_{\gT},W f,h)=\E_{x\sim\gD}\E_{\{c_i\}_{i=1}^k\iidsim\rT^{\not y_x}}\ell(f(x),\mu_{y_x},\{\mu_{c_i}\}_{i=1}^k).
\end{equation*}

We now prove the case of $\Lsup'$ in Theorem \ref{thm:optimality-of-InfoNCE}.
\begin{proof}
Since $\rT$ is uniform, let $C=|\gT|$. Then we have $\overline{\mu}=\E_{c\sim\rT}\mu_c=\frac{1}{C}\sum_{c\in\gT}\mu_c$, and
\begin{equation*}
\begin{aligned}
\Lsup'=& \E_{x\sim\gD}\E_{\{c_i\}_{i=1}^k\iidsim\rT^{\not y_x}}\log\left[1+\sum_{i=1}^k\exp\left(f(x)^\top(\mu_{c_i}-\mu_{y_x})\right)\right]\\
 \ge& \E_{x\sim\gD}\log\left[1+k\exp\left(f(x)^\top\left(\tfrac{C}{C-1}\overline{\mu}-\tfrac{C}{C-1}\mu_{y_x}\right)\right)\right]\\
\ge& \frac{1}{C}\sum_{c\in\gT}\log\left[1+k\exp\left(\tfrac{C}{C-1}\mu_c^\top(\overline{\mu}-\mu_c)\right)\right]\\
\ge & \log\left[1+k\exp\left(\tfrac{1}{C-1}\sum_{c\in\gT}\mu_c^\top(\overline{\mu}-\mu_c)\right)\right]\\
\ge&\log\left[1+k\exp\left(-\tfrac{CR^2}{C-1}\right)\right].
\end{aligned}
\end{equation*}
    As in the proof of Lemma \ref{lemma:infonce-ce}, equality in the last inequality requires $\overline{\mu}=0$ and $\|\mu_c\|=R$, which imply $f(x)=\mu_c$ almost surely under $\gD^c$. The function $(z_1,\ldots,z_k)\mapsto\log(1+\sum_i e^{z_i})$ is strictly convex, so equality in the first inequality requires $\mu_c^\top\mu_{c'}$ to be constant over $c'\ne c$. Their sum is $-R^2$, hence each equals $-R^2/(C-1)$. Conversely, the feasible simplex Gram matrix in the proof of Lemma \ref{lemma:infonce-ce} attains equality throughout, completing the proof.
\end{proof}

\begin{theorem}
Consider the supervised spectral loss $\Lsup=\E_{x\sim\gD}\E_{\{c_i\}_{i=1}^k\iidsim\rT^{\not y_x}}\ell(f(x),\mu_{y_x},\{\mu_{c_i}\}_{i=1}^k)$, that is, $\ell(f(x),\mu_{y_x},\{\mu_{c_i}\}_{i=1}^k)=-2f(x)^\top \mu_{y_x}+\frac{1}{k}\sum_{i=1}^k (f(x)^\top\mu_{c_i})^2$. Then $f$ minimizes $\Lsup$ over all representation functions satisfying $\|f(x)\|\le R$ if and only if $f(x)=\mu_{y_x}^*$ almost surely, where $\{\mu_{c}^*\}_{c\in\gT}$ minimizes $\E_{c\sim\rT}\left[-2\mu_c^\top\mu_c+\E_{c'\sim\rT^{\not c}}(\mu_c^\top\mu_{c'})^2\right]$ subject to $\|\mu_c\|\le R$.
\end{theorem}
\begin{proof}
For the spectral contrastive loss, we have
\begin{equation*}
\begin{aligned}
\Lsup=& \E_{x\sim\gD}\E_{\{c_i\}_{i=1}^k\iidsim\rT^{\not {y_x}}}\left[-2f(x)^\top \mu_{y_x}+\frac{1}{k}\sum_{i=1}^k (f(x)^\top\mu_{c_i})^2\right]\\
=& \E_{x\sim\gD}\left[-2f(x)^\top \mu_{y_x}+\E_{c\sim\rT^{\not {y_x}}} (f(x)^\top\mu_{c})^2\right]\\
\ge&\E_{c\sim\rT}\left[-2\mu_c^\top\mu_c+\E_{c'\sim\rT^{\not c}}(\mu_c^\top\mu_{c'})^2\right].
\end{aligned}
\end{equation*}
    Denote the last expression by $Q(\{\mu_c\})$. It has a minimizer on the compact set $\|\mu_c\|\le R$, and the corresponding class-constant representations attain the lower bound. Moreover,
    \begin{equation*}
        \Lsup-Q(\{\mu_c\})=\E_{c\sim\rT}\E_{x\sim\gD^c}\E_{c'\sim\rT^{\not c}}\left[((f(x)-\mu_c)^\top\mu_{c'})^2\right].
    \end{equation*}
    Any minimizer must therefore minimize $Q$ and satisfy $(f(x)-\mu_c)^\top\mu_{c'}=0$ almost surely for every $c'\ne c$. If a class has nonzero within-class variance, then $\|\mu_c\|<R$ and there exists a nonzero $v$ orthogonal to all $\mu_{c'}$, $c'\ne c$. Replacing $\mu_c$ by $\mu_c+sv$ for sufficiently small $s$ preserves feasibility and all cross-class inner products. Choosing the sign of $s$ so that $2s\mu_c^\top v+s^2\|v\|^2>0$ strictly decreases $Q$ through its term $-2\rT(c)\|\mu_c\|^2$, a contradiction. Thus every minimizer is class-constant almost surely, proving both directions.
\end{proof}

\section{One Negative Sample Case}
\label{appendix:more-results}
% \subsection{One Negative Sample Case}
To better understand the supervised loss incurred by contrastive learning, we also explore the special case where only one negative sample is used (i.e., $k=1$). Negative sampling is also used in Word2Vec \citep{mikolov2013distributed}. Here, we consider the contrastive objective with one negative sample, defined as follows:
\begin{equation}
\Lcon(f,\gD;\ell)=\E_{\substack{x,x^+\sim\Dpos\\x^-\sim\Dneg}}\ell(f(x),f(x^+),f(x^-)).
\end{equation}
In this case, the incurred loss $h$ in Theorem \ref{thm:bound-for-contrastive-losses} can be precisely characterized when class collision results in a constant loss value. Let $\pi=\E_{c,c'\iidsim\rT}\mathbb{I}(c=c')$. For losses satisfying $\ell(v,u,u)=\ell(0)$ for all $\|v\|,\|u\|\le R$, where $\ell(0)$ is constant, we can derive a similar bound to that of \citet{saunshi2019theoretical} without relying on the assumption of conditional independence of positive pairs.
\begin{proposition}
\label{bound-for-binary}
Under the notation and conditions in Theorem \ref{thm:bound-for-contrastive-losses}, suppose $k=1$ and $\ell(v,u,u)=\ell(0)$ is constant for $\|v\|,\|u\|\le R$. Then, for all $f\in\gF$, there exist a linear classifier $W\in\sR^{|\gT|\times d}$ and a supervised loss $h$ such that
\begin{equation}
\Lsup(\gD_{\gT}, Wf, h)\le \Lcon(f)+\eta\cdot \Disc(f)-\pi\cdot\ell(0),
\end{equation}
where the incurred loss $h(w_y^\top f(x),\{w_c^\top f(x)\}_{c\ne y})=(1-\rT(y))\E_{c\sim\rT^{\not y}}\ell(f(x),w_y,w_c)$.
\end{proposition}
\begin{proof}
    Let $w_c=\mu_c=\E_{x\sim\gD^c}f(x)$. Splitting the mean-classifier loss according to whether the negative class equals the anchor class gives
    \begin{equation*}
        \begin{aligned}
            \E_{x\sim\gD}\E_{c\sim\rT}\ell(f(x),\mu_{y_x},\mu_c)
            &=\E_{x\sim\gD}\left[(1-\rT(y_x))\E_{c\sim\rT^{\not y_x}}\ell(f(x),\mu_{y_x},\mu_c)+\rT(y_x)\ell(0)\right]\\
            &=\Lsup(\gD_{\gT},Wf,h)+\pi\ell(0).
        \end{aligned}
    \end{equation*}
    Theorem \ref{thm:bound-for-contrastive-losses} bounds the left-hand side by $\Lcon(f)+\eta\Disc(f)$, completing the proof.
\end{proof}

\noindent\textbf{Multi-Class Binary Cross Entropy.} For the InfoNCE loss with only one negative sample, we have the incurred loss
\begin{equation}
\label{eq:multi-class-binary-ce}
h(Wf(x),y)=(1-\rT(y))\E_{c\sim\rT^{\not y}}\log\left[1+e^{f(x)^\top(w_c-w_y)}\right],
\end{equation}
which can be viewed as a \textit{multi-class binary cross entropy} that trains the binary classifiers $\{w_y-w_c\}_{c\neq y}$ with weighted logistic losses for each sample $x$ drawn from $\gD^y$.

Taking $\ell(0)=\log2$ in Proposition \ref{bound-for-binary} gives the deterministic InfoNCE bound $\Lsup(\gD_{\gT},Wf,h)\le\Lcon(f)+\eta\Disc(f)-\pi\log2$ for the loss in Equation \ref{eq:multi-class-binary-ce}.

For a direct classification guarantee, let $h_1$ denote the $k=1$ InfoNCE-incurred loss in Theorem \ref{thm:bound-for-contrastive-losses}, including target-class negatives, and set $p_{\min}=\min_{c\in\gT}\rT(c)>0$. With any fixed rule for breaking ties, the classification error under $\gD_\gT$ satisfies
\begin{equation}
    \label{eq:infonce-classification-bound}
    \sP_{(x,y)\sim\gD_\gT}\left[\argmax_{c\in\gT}w_c^\top f(x)\ne y\right]
    \le\frac{\Lsup(\gD_\gT,Wf,h_1)-\pi\log2}{p_{\min}\log2}.
\end{equation}
This is the case $k=1$ of the first bound in Corollary \ref{cor:classification-error}, since then $\E_{y\sim\rT}h_0(y)=\pi\log2$ and $\omega_1(y)=\log2\cdot\min_{c\ne y}\rT(c)\ge p_{\min}\log2$, and that corollary also bounds its numerator for the mean classifier of $\widehat{f}$.

\section{Experiments}
\label{sec:appendix-exps}

\subsection{Experimental Details}
In this section, we present the training details and experimental results across various tasks, encompassing unsupervised contrastive learning and supervised learning on CIFAR-10, CIFAR-100 \citep{krizhevsky2009learning}, and ImageNet-100 (the first 100 classes of ImageNet) \citep{deng2009imagenet}. The chosen network architectures are ResNet-18 \citep{he2016deep} for CIFAR-10/CIFAR-100, and ResNet-50 for ImageNet-100. All models are trained for 200 epochs using the SGD optimizer and a cosine decay learning rate schedule, with an initial learning rate of 0.2, a batch size of 512, and a weight decay of 5e-4.
% The GPUs are NVIDIA GeForce RTX 3090Ti.
\begin{table}[!t]
\centering
% \scriptsize
% \setlength{\tabcolsep}{1.5mm}
\caption{Linear probing accuracies (\%) on CIFAR-10, CIFAR-100, and ImageNet-100 using unsupervised contrastive learning with the InfoNCE loss and the relaxed k-means loss, i.e., $\widehat{L}_{\mathrm{con}}+\lambda\cdot \widehat{L}_{\mathrm{k\text{-}means}}(f)$, where $\lambda\in\{0, 0.01, 0.02, 0.05, 0.1, 0.2, 0.5, 1\}$ and $\tau=0.5$.}
\label{tab:relaxed-loss}
\begin{sc}
\begin{tabular}{c|cccc}
\toprule
$\lambda$       & 0                     & 0.01                  & 0.02                  & 0.05                  \\
\midrule
CIFAR-10     & 91.73 $\pm$ 0.06          & 91.68 $\pm$ 0.05          & \textbf{91.88 $\pm$ 0.07} & \textbf{91.85 $\pm$ 0.08} \\
CIFAR-100    & 77.32 $\pm$ 0.08          & \textbf{77.37 $\pm$ 0.05} & \textbf{77.35 $\pm$ 0.10} & \textbf{77.54 $\pm$ 0.11} \\
ImageNet-100 & 81.09 $\pm$ 0.20            & \textbf{81.22 $\pm$ 0.18}   & \textbf{81.40 $\pm$ 0.23}   & \textbf{81.28 $\pm$ 0.48}   \\
\midrule\midrule
$\lambda$       & 0.1                   & 0.2                   & 0.5                   & 1                     \\
\midrule
CIFAR-10     & \textbf{91.80 $\pm$ 0.04} & 89.84 $\pm$ 0.08          & 86.57 $\pm$ 0.16          & 81.51 $\pm$ 0.20            \\
CIFAR-100    & \textbf{77.37 $\pm$ 0.06} & 75.04 $\pm$ 0.07          & 69.03 $\pm$ 0.25          & 65.87 $\pm$ 1.51            \\
ImageNet-100 & \textbf{81.63 $\pm$ 0.24}   & \textbf{82.10 $\pm$ 0.14}   & \textbf{83.09 $\pm$ 0.17}   & 65.49 $\pm$ 2.46  \\
\bottomrule
\end{tabular}
\end{sc}
\end{table}

\begin{table}[!t]
\centering
\setlength{\tabcolsep}{1.6mm}
\caption{Illustration of $\frac{1}{\rT(c)}$ and $\mu_c^\top w_c$ on long-tailed CIFAR-10 with an imbalance ratio of 10.}
\label{tab:val-specl}
\begin{sc}
\begin{tabular}{c|cccccccccc}
\toprule
Class Index & 1 & 2  & 3 & 4 & 5 & 6 & 7 & 8 & 9 & 10\\
\midrule
$\frac{1}{\rT(c)}$  & 4.09 &	5.28	&6.82	&8.81	&11.38	&14.69	&18.97	&24.50&	31.68&	40.86	  \\
$\mu_c^\top w_c$ (train) & 4.08	&5.21	&6.50	&8.12	&10.54&	12.36	&17.17&	21.79&	27.12&	34.11\\
$\mu_c^\top w_c$ (test) & 4.08	&	5.21	&	6.51	&	8.06	&	10.46	&	12.37	&	17.17	&	21.79	&	27.12	&	34.11\\
\bottomrule
\end{tabular}
\end{sc}
\end{table}

\subsection{More Experimental Results}
\paragraph{More Results on the Relaxed K-Means Loss.}
\label{sec:more-results-relaxed-kmeans}
We conduct experiments with independent runs and report standard deviations to illustrate the performance of the relaxed k-means loss. For Table \ref{tab:relaxed-loss}, we run 3 trials with random seeds (0, 1, 2) for each $\lambda\in\{0, 0.01, 0.02, 0.05, 0.1, 0.2, 0.5, 1\}$. As is evident in the table, the performance in linear probing deteriorates when $\lambda$ is too large, suggesting inherent limitations of the relaxed k-means loss. This could be attributed to the fact that it only focuses on reducing distances between representations within a specific range rather than rigorously shrinking the distance to cluster centers. Consequently, this relaxation would inevitably cause representations of semantically inconsistent samples to become close, particularly in the early stage of training. Despite these limitations, we still observe an improvement, particularly on ImageNet-100. 
% This underscores the feasibility of minimizing the internal discrepancy in this manner. We hope that the relaxed k-means loss can stimulate new insights into unsupervised representation learning.

% \paragraph{More Results for Class Imbalance.} The complete experimental results for long-tailed CIFAR-10 and long-tailed CIFAR-100 are presented in Table~\ref{tab:complete-imbalanced-learning}.

\subsection{Validation of Theorem \ref{thm:balanced-SpeCL}}
\label{sec:val-specl-imbalance}
In the proof of Theorem \ref{thm:balanced-SpeCL}, a global minimizer is obtained by minimizing $\lVert (\sqrt{\rT(c)}\mu_c)^\top_{c\in\gT}(\sqrt{\rT(c)}w_c)_{c\in\gT}-\mI\rVert_F^2$ over class means and classifier weights, where $\mu_c=\E_{x\sim\gD^c}f(x)$. For $s>0$, let $\hat{\mu}_c=s\sqrt{|\gT|\rT(c)}\mu_c$ and $\hat{w}_c=\frac{1}{s}\sqrt{|\gT|\rT(c)}w_c$. The matrix objective then equals $\lVert (\frac{1}{\sqrt{|\gT|}}\hat{\mu}_c)^\top_{c\in\gT}(\frac{1}{\sqrt{|\gT|}}\hat{w}_c)_{c\in\gT}-\mI\rVert_F^2$ for the class-balanced case. Both distributions attain the same minimum, so this scaling maps corresponding global minimizers while preserving cosine similarities between their nonzero features and classifier weights.

To illustrate this point, we conduct experiments on long-tailed CIFAR-10 using ResNet-18 with a feature dimension of 512, and $|\gT|=10$. Since $d\ge|\gT|$, we can derive that the global minimizer of $\lVert (\sqrt{\rT(c)}\mu_c)^\top_{c\in\gT}(\sqrt{\rT(c)}w_c)_{c\in\gT}-\mI\rVert_F^2$ satisfies $\mu_c^\top w_c=\frac{1}{\rT(c)}$, $\mu_c^\top w_{c'}=0$ for $c\neq c'$. The results are presented in Table \ref{tab:val-specl}. As observed, the logits $\mu_c^\top w_c$ of the representation model roughly align with $\frac{1}{\rT(c)}$ across the board.

\vskip 0.2in
\bibliography{preference}

\end{sloppypar}
\end{document}